%% file: paper.tex
\pdfoutput=1

\documentclass{article}
\usepackage{ijcai15}
\usepackage{paralist}
\usepackage{enumitem}
\usepackage{amssymb}
\usepackage{amsmath}
\usepackage{amsthm}
\usepackage{latexsym}
\usepackage{xspace}
\usepackage{tikz}
\usepackage{multirow}
\usetikzlibrary{matrix}
\usepackage{stmaryrd}
\usepackage{times}

\usepackage{macros}
\usepackage{thmtools,thm-restate}
\declaretheorem{theorem}
\declaretheorem[sibling=theorem]{proposition}

\declaretheorem[style=definition,sibling=theorem,qed=$\diamond$]{definition}

\title{Computing Horn Rewritings of Description Logics Ontologies}
\author{Mark Kaminski \and Bernardo Cuenca Grau\\
Department of Computer Science, University of Oxford, UK}

\begin{document}

\maketitle

\begin{abstract}
  We study the problem of rewriting an ontology $\Ont_1$ expressed in
  a DL $\mathcal{L}_1$ into an ontology $\Ont_2$ in a Horn DL
  $\mathcal{L}_2$ such that $\Ont_1$ and $\Ont_2$ are equisatisfiable
  when extended with an arbitrary dataset.  Ontologies that admit such
  rewritings are amenable to reasoning techniques ensuring
  tractability in data complexity.  After showing undecidability
  whenever $\mathcal{L}_1$ extends $\mathcal{ALCF}$, we focus on
  devising efficiently checkable conditions that ensure existence of a
  Horn rewriting. By lifting existing techniques for rewriting
  Disjunctive Datalog programs into plain Datalog to the case of
  arbitrary first-order programs with function symbols, we identify a
  class of ontologies that admit Horn rewritings of polynomial size.
  Our
  experiments indicate that many real-world ontologies satisfy our
  sufficient conditions and thus admit polynomial Horn
  rewritings.
\end{abstract}

\input{intro}

\input{preliminaries}
\input{definitions}
\input{markability}
\input{transposition}
\input{related}
\input{evaluation}
\input{future}

\section*{Acknowledgments}

Work supported by the Royal Society, the EPSRC projects
Score!, $\text{MaSI}^3$ and DBOnto, and the FP7 project \mbox{Optique}.


\bibliographystyle{named}
\bibliography{bib,paper}

\clearpage
\onecolumn
\appendix

\input{proofs-hornrewritability}
\input{proofs-transposition}
\input{proofs-xi}
\input{proofs-rewriting}

\end{document}

%% file: intro.tex
\section{Introduction}

Reasoning over ontology-enriched datasets is a 
key requirement in many applications of semantic technologies.
Standard reasoning tasks are, however,
of high worst-case complexity. Satisfiability checking 
is \textsc{2NExpTime}-complete for the description logic (DL) $\mathcal{SROIQ}$  underpinning
the standard ontology language OWL 2 and \textsc{NExpTime}-complete for $\mathcal{SHOIN}$, which underpins
OWL DL \cite{DBLP:conf/kr/Kazakov08}. Reasoning is also
\coNP-hard with respect to \emph{data complexity}---a key 
measure of complexity for applications involving
large amounts of instance data \cite{DBLP:conf/ijcai/HustadtMS05}.

Tractability in data complexity is typically associated with
\emph{Horn} DLs, where 
ontologies correspond to first-order Horn clauses \cite{DBLP:conf/ijcai/OrtizRS11,DBLP:conf/ijcai/HustadtMS05}.
The more favourable computational properties of Horn DLs
make them a natural choice for data-intensive applications, but
they also come at the expense of a loss in expressive power.
In particular, Horn DLs 
cannot capture disjunctive axioms, i.e., statements such as
``every $X$ is either a $Y$ or a $Z$''. Disjunctive axioms are common
in real-world ontologies, like the NCI Thesaurus
or the ontologies underpinning
the European Bioinformatics Institute (EBI) linked data platform.\footnote{http://www.ebi.ac.uk/rdf/platform} 
 
In this paper we are interested in \emph{Horn rewritability} of description
logic ontologies; that is, whether an ontology $\Ont_1$ expressed in a DL
$\mathcal{L}_1$ can be rewritten into an ontology $\Ont_2$ in a Horn 
DL $\mathcal{L}_2$ 
such that $\Ont_1$ and $\Ont_2$ are equisatisfiable when extended 
with an arbitrary dataset. Ontologies that admit such Horn rewritings 
are amenable to more efficient reasoning techniques that ensure tractability
in data complexity.

Horn rewritability of DL ontologies is strongly related to 
the rewritability of Disjunctive Datalog programs into Datalog, where  
both the source and target languages for rewriting
are function-free.
\citeA{KaminskiNCG14AAAI} 
characterised 
Datalog rewritability of Disjunctive Datalog programs
in terms of linearity: a restriction that requires each rule to contain at most
one body atom that is IDB (i.e.,  
whose predicate also occurs in head position in the program). It was shown that 
every  linear Disjunctive
Datalog program
can be rewritten into plain Datalog (and vice versa)
by means of \emph{program transposition}---a polynomial
transformation in which rules are ``inverted''
by shuffling all IDB atoms between head and body while at the same time
replacing their predicates by auxiliary ones. 
Subsequently, \citeA{KaminskiNCG14RR} proposed the class
of \emph{markable} Disjunctive Datalog programs, where
the linearity requirement is relaxed 
so that it applies only to a subset
of ``marked'' atoms. 
Every markable program can be
polynomially rewritten into Datalog  
by exploiting a variant of transposition
where only marked atoms are affected.

Our contributions in this paper are as follows.
In Section \ref{sec:horn-rewritability}, we 
show undecidability of Horn rewritability whenever the
input ontology is expressed in $\mathcal{ALCF}$.
This is in consonance 
with the related undecidability results  
by \citeA{BienvenuCLW14} and \citeA{Lutz:2012ug}
for Datalog rewritability and non-uniform data complexity for $\mathcal{ALCF}$ ontologies.

In Section \ref{sec:markability-transposition}, we lift the
markability condition and the transposition transformation in
\cite{KaminskiNCG14RR} for Disjunctive Datalog to arbitrary
first-order programs with function symbols.  We then show that all
markable first-order programs admit Horn rewritings of polynomial
size.  This result is rather general and has potential implications in
areas such as theorem proving \cite{Robinson:2001:HAR:778522} and
knowledge compilation \cite{Kcomp}.
      
The notion of markability for first-order programs
can be seamlessly adapted to ontologies via the standard
FOL translation of DLs \cite{DBLP:conf/dlog/2003handbook}. 
This is, however, of limited practical value
since Horn programs obtained via transposition may not be expressible using 
standard DL constructors. In Section \ref{sec:markability-DLs},
we introduce an alternative satisfiability-preserving
translation from $\mathcal{ALCHIF}$ ontologies into first-order programs
and show in
Section \ref{sec:rewriting} that the corresponding transposed programs
can be translated back into Horn-$\mathcal{ALCHIF}$ ontologies.
Finally,
we focus on complexity and
show that reasoning  over markable $\mathcal{L}$-ontologies
is 
\textsc{ExpTime}-complete in combined complexity and
\textsc{PTime}-complete w.r.t.\  data for each DL
$\mathcal{L}$ between $\mathcal{ELU}$ and $\mathcal{ALCHIF}$.
All our results immediately extend to DLs with transitive
roles (e.g., $\mathcal{SHIF}$) by exploiting standard transitivity
elimination techniques \cite{DBLP:conf/dlog/2003handbook}.

We have implemented markability checking and evaluated 
our techniques on a large ontology repository.
Our results indicate that many real-world ontologies
are markable and thus admit Horn rewritings of polynomial size.

The proofs of all our results are delegated to the appendix.


%% file: preliminaries.tex
\section{Preliminaries}

We assume standard first-order syntax and semantics.
We treat the universal truth $\top$ and falsehood $\bot$ symbols as
well as equality ($\equality$) as ordinary predicates of arity one
($\top$ and $\bot$) and two ($\equality$),
the meaning of which will be axiomatised.

\smallskip
\noindent
\textbf{Programs }
A \emph{first-order rule} (or just a rule) is a sentence 
$$\forall \vec{x} \forall
\vec{z}.[\varphi(\vec{x},\vec{z}) \rightarrow \psi(\vec{x})]$$
where
variables $\vec{x}$ and $\vec{z}$ are disjoint,
$\varphi(\vec{x},\vec{z})$ is a conjunction of distinct 
atoms over $\vec x \cup \vec y$, and $\psi(\vec{x})$ is 
a disjunction of distinct atoms over $\vec x$. Formula
$\varphi$ is the \emph{body} of $r$, and $\psi$ is the \emph{head}.
Quantifiers are omitted for brevity, and safety is assumed 
(all variables in the rule occur in the
body).  
We define the following sets of rules for a finite signature $\Sigma$:
\begin{inparaenum}[\it(i)]
\item $\toppart{\P_{\Sigma}}$ consists of a rule $P(x_1, \ldots, x_n)
  \rightarrow \top(x_i)$ for each predicate $P \in \Sigma$ and each $1
  \leq i \leq n$ and a rule $\,\to\top(a)$ for each constant $a \in
  \Sigma$;
\item $\botpart{\P_{\Sigma}}$ consists of the rule having $\bot(x)$ in the body and an empty head; and
\item $\eqpart{\P_{\Sigma}}$ consists of the standard axiomatisation
  of $\equality$ as a congruence over $\Sigma$.\footnote{
    Reflexivity of $\equality$ is axiomatised by the safe
    rule \mbox{$\top(x) \rightarrow x\,{\equality}\,x$}.}
\end{inparaenum}
A \emph{program} is a finite set of rules $\P = \P_0 \cup
\toppart{\P_{\Sigma}} \cup \botpart{\P_{\Sigma}} \cup
\eqpart{\P_{\Sigma}}$ with $\Sigma$ the signature of $\P_0$, where we
assume w.l.o.g.\ that the body of each rule in $\P_0$ does not mention
$\bot$ or $\equality$, and the head is non-empty and does not mention
$\top$.
%
We omit $\Sigma$ for
the components of $\P$
and write $\toppart{\P}$, $\botpart{\P}$ and $\eqpart{\P}$.
A rule is \emph{Horn} if its head consists of at most one atom, and
a program is Horn if so are all of its rules.
Finally, a \emph{fact} is a ground, function-free
atom, and a \emph{dataset} is a finite set of facts.

\begin{table*}[t]
\begin{footnotesize}
\begin{displaymath}
\begin{array}{@{}l@{\quad}r@{\;}l@{\qquad}l@{}}
    T1.  & \bigsqcap_{i=1}^n A_i  & \sqsubseteq \bigsqcup_{j=1}^m C_j                 & \bigwedge_{i=1}^n A_i(x) \rightarrow \bigvee_{j=1}^m C_j(x) \\
    T2.  & \exists R.A   & \sqsubseteq C                 & \Inv{R}{x}{y} \wedge A(y) \rightarrow C(x) \\
    T3. & A             & \sqsubseteq  \exists R.B & A(x) \rightarrow \Inv{R}{x}{f(x)};~ A(x) \rightarrow B(f(x)) \\
    T4.  & A   & \sqsubseteq \forall R.C                 & A(x) \wedge \Inv{R}{x}{y}  \to C(y) \\
    T5.  & S             & \sqsubseteq R                 & S(x,y) \rightarrow \Inv{R}{x}{y} \\
    T6. & A             & \sqsubseteq  \atmostq{1}{R}{B} & A(z) \wedge \Inv{R}{z}{x_1} \wedge \Inv{R}{z}{x_2} \wedge B(x_1) \wedge B(x_2) \rightarrow  x_1 \equality x_2  
\end{array} 
\end{displaymath}
\caption{Normalised DL axioms. $A,B$ are named 
or $\top$;  $C$ named or $\bot$; 
role $S$ is named and $R$ is a (possibly inverse) role. 
}
\label{tab:RL}
\end{footnotesize}
\end{table*} 

\smallskip
\noindent
\textbf{Ontologies }
We assume familiarity with DLs and ontology
languages
\cite{DBLP:conf/dlog/2003handbook}.
A DL signature $\Sigma$ consists of disjoint countable sets of concept names
$\Sigma_C$ and role names $\Sigma_R$. A role is an element of
$\Sigma_R \cup \{R^- \mid R \in \Sigma_R\}$. The function
$\mathsf{inv}$ is defined over roles as follows, where
$R \in \Sigma_R$: $\mathsf{inv}(R) = R^-$ and $\mathsf{inv}(R^-) = R$.
W.l.o.g., we consider normalised axioms as on the left-hand side of
Table~\ref{tab:RL}. 

An $\mathcal{ALCHIF}$ ontology $\Ont$ is a finite set of DL axioms
of type T1-T6 in Table \ref{tab:RL}. An ontology is \emph{Horn} if
it contains no axiom of type T1 satisfying $m \geq 2$.
Given $\Ont$, we denote with $\sqsubseteq^{*}$ the minimal 
reflexive and transitive relation over
roles in $\Ont$ such that $R_1 \sqsubseteq^{*} R_2$ and $\mathsf{inv}(R_1) \sqsubseteq^{*} \mathsf{inv}(R_2)$
hold whenever $R_1 \sqsubseteq R_2$ is an axiom in $\Ont$.

We refer to the DL where only axioms of type T1-T3 are available and
the use of inverse roles is disallowed as $\mathcal{ELU}$. The logic $\mathcal{ALC}$
extends  $\mathcal{ELU}$ with axioms T4. We then 
use standard naming conventions for DLs based on the presence of
inverse roles $(\mathcal{I})$, axioms T5 ($\mathcal{H}$) and axioms T6 ($\mathcal{F}$).
Finally, an ontology is $\mathcal{EL}$ if it is both $\mathcal{ELU}$ and Horn.


Table \ref{tab:RL} also provides the standard translation $\pi$ from
normalised axioms into first-order rules, where $\Inv{R}{x}{y}$ is defined as
$R(x,y)$ if $R$ is named and as $S(y,x)$ if $R = S^-$.
We define $\pi(\Ont)$ as the smallest program containing $\pi(\alpha)$
for each axiom $\alpha$ in $\Ont$.
Given a dataset $\Dat$, 
we say that $\Ont \cup \Dat$ is satisfiable iff so is $\pi(\Ont) \cup \Dat$ in first-order logic.


%% file: definitions.tex
\section{Horn Rewritability} \label{sec:horn-rewritability}

Our focus is on satisfiability-preserving rewritings. 
Standard reasoning tasks in description logics
are reducible to unsatisfiability checking \cite{DBLP:conf/dlog/2003handbook}, which makes our
results practically relevant.
We start by formulating our notion of rewriting in general terms.

\begin{definition}
Let $\mathcal{F}$ and $\mathcal{F}'$ be sets of rules.
We say that $\mathcal{F}'$ is 
a \emph{rewriting} of $\mathcal{F}$
if it holds that $\mathcal{F} \cup \Dat$ is satisfiable iff so is $\mathcal{F}' \cup \Dat$
for each dataset $\Dat$ over predicates from
$\mathcal{F}$.
\end{definition}

We are especially 
interested in computing Horn rewritings of  ontologies---that is, rewritings
where the given ontology $\Ont_1$ is expressed in a DL $\mathcal{L}_1$
and the rewritten ontology $\Ont_2$ is in a Horn DL $\mathcal{L}_2$
(where preferably $\mathcal{L}_2 \subseteq \mathcal{L}_1$).
This is not possible in general: satisfiability checking
is \coNP-complete in data complexity even for the basic logic $\mathcal{ELU}$ \cite{DBLP:conf/lpar/KrisnadhiL07}, whereas
data complexity is tractable even for highly expressive Horn languages such as
Horn-$\mathcal{SROIQ}$ \cite{DBLP:conf/ijcai/OrtizRS11}. 
Horn rewritability for DLs can be formulated as a decision problem as
follows:
\begin{definition}
The $(\mathcal{L}_1,\mathcal{L}_2)$-Horn rewritability problem for DLs $\mathcal{L}_1$ and $\mathcal{L}_2$
is to decide whether a given $\mathcal{L}_1$-ontology  admits a rewriting
expressed in Horn-$\mathcal{L}_2$. 
\end{definition}
 
Our first result establishes undecidability whenever the
input ontology contains at-most cardinality restrictions and thus equality.
This result is in consonance with the related undecidability results  
by \citeA{BienvenuCLW14} and \citeA{Lutz:2012ug}
for Datalog rewritability and non-uniform data complexity for $\mathcal{ALCF}$ ontologies.
\begin{restatable}{theorem}{hornrewundecidable}
$(\mathcal{L}_1,\mathcal{L}_2)$-Horn rewritability is undecidable for
$\mathcal{L}_1 = \mathcal{ALCF}$ and $\mathcal{L}_2$ any 
DL between $\mathcal{ELU}$
and $\mathcal{ALCHIF}$. This result holds under the assumption that
\textsc{PTime}$\neq$\textsc{NP}.
\end{restatable}
 
Intractability results in data complexity 
rely on the ability of non-Horn DLs to encode \coNP-hard problems,
such as non-3-colourability
\cite{DBLP:conf/lpar/KrisnadhiL07,DBLP:conf/ijcai/HustadtMS05}.  In
practice, however, it can be expected that ontologies do not encode
such problems.  Thus, our focus from now onwards will be on
identifying classes of ontologies that admit (polynomial size) Horn
rewritings.

%% file: markability.tex
\section{Program Markability and Transposition} \label{sec:markability-transposition}

In this section, we introduce the class of \emph{markable} programs and show that
every markable program can be rewritten into a Horn program by means of a polynomial 
 transformation, which we refer to as \emph{transposition}.
Roughly speaking, transposition inverts the rules in a program $\P$
by moving certain atoms from head to body and vice versa while 
replacing their corresponding predicates with fresh ones.
Markability of $\P$ ensures that we can pick a set of predicates (a \emph{marking})
such that, by shuffling only atoms with a marked predicate, we obtain a Horn rewriting of $\P$.
Our results in this section generalise the results
by \citeA{KaminskiNCG14RR}
for Disjunctive Datalog to first-order programs with function symbols.

To illustrate our definitions throughout this section, we
use an example
program $\Pex$ consisting of the following rules:
  \begin{align*} 
    A(x)\to B(x)  && B(x)\to C(x)\lor D(x)  \\
    C(x)\to \bot(x)  &&  D(x)\to C(f(x)) 
  \end{align*}
 
\smallskip
\noindent
\textbf{Markability}. The notion of markability involves a partitioning of the program's predicates 
into \emph{Horn} and
\emph{disjunctive}. Intuitively, the former are those 
whose extension
for all datasets depends only on the Horn rules in the program, whereas the latter are those
whose extension may depend on a disjunctive rule. This intuition can be
formalised 
using the standard notion of a dependency graph in Logic Programming.

\begin{definition}\label{def:WL}
  The \emph{dependency graph} $G_\DDP=(V,E,\mu)$ of a program $\DDP$ is the
  smallest edge-labeled digraph such that:
  \begin{inparaenum}[\it(i)]
  \item $V$ contains all predicates in $\DDP$;
  \item $r\in\mu(P,Q)$ whenever 
    $r\in\DDP$, 
    $P$ is in the body of $r$, and $Q$ is in the head of $r$;
    and
  \item $(P,Q)\in E$ whenever $\mu(P,Q) \neq \emptyset$.
  \end{inparaenum}
  A predicate $Q$ \emph{depends on $r\in\DDP$} if $G_\DDP$ has a path
  ending in $Q$ and involving an $r$-labeled edge. Predicate $Q$ is
  \emph{Horn} if it depends only on Horn rules; otherwise, $Q$ is
  \emph{disjunctive}.
 \end{definition}
For instance, predicates $C$, $D$, and $\bot$ are disjunctive in our example program $\Pex$, whereas
$A$ and $B$ are Horn.
We can now introduce the notion of a \emph{marking}---a subset of the disjunctive
predicates in a program $\P$ ensuring that the transposition of $\P$ 
where only marked atoms are
shuffled between head and body results in a Horn program.

\begin{definition}\label{def:marking}
  A \emph{marking} of a program $\P$ is a
  set\/ $M$ of disjunctive predicates in $\P$ satisfying the following properties, where
  we say that an atom is marked if its predicate is in $M$:
  \begin{inparaenum}[\it(i)]
  \item each rule in $\P$ 
    has at most one marked body atom;
  \item each rule in $\P$ 
    has at most one unmarked head atom; and
  \item if\/ $Q\in M$ and $P$ is reachable from $Q$ in $G_\P$, then
    $P\in M$.
  \end{inparaenum}
  We say that a program is
  \emph{markable} if it admits a marking.
\end{definition}

Condition \emph{(i)} in Def.\ \ref{def:marking} ensures
that at most one atom is moved from body to head during
transposition. Condition \emph{(ii)} ensures that
all but possibly one head atom are moved to the body. Finally,
condition \emph{(iii)} requires that all predicates depending on a
marked predicate are also marked.  We can observe that our example
program $\Pex$ admits 
two markings: $M_1=\set{C,\bot}$ and $M_2=\set{C,D,\bot}$.

Markability can be efficiently checked via a 2-SAT reduction, where
we assign to each predicate $Q$ in $\P$ a propositional
variable $X_Q$ and encode the constraints in  Def.\  \ref{def:marking}
as 2-clauses. For each rule 
$\varphi \wedge \bigwedge_{i=1}^n P_i(\vec s_i) \to \bigvee_{j=1}^m Q_j(\vec t_j)$, with $\varphi$ the conjunction
of all Horn atoms in the rule head,
we include clauses \emph{(i)} 
$\neg X_{P_i} \vee \neg X_{P_j}$ for all 
$1 \leq i <j \leq n$, which enforce at most one body atom to be marked;
\emph{(ii)} $X_{Q_i} \vee X_{Q_j}$ for $1 \leq i < j \leq m$, which ensure that
at most one head atom is unmarked; and \emph{(iii)} $\neg X_{P_i} \vee X_{Q_j}$ for $1 \leq i \leq n$
and $1 \leq j \leq m$, which close markings under rule dependencies.
Each model of the resulting clauses yields a marking of $\P$.

\smallskip
\noindent
\textbf{Transposition.} Before defining transposition, 
we
illustrate the main intuitions using 
program $\Pex$ and marking $M_1$. 

The first step to transpose $\Pex$ is to introduce fresh unary
predicates $\overl C$ and $\overl \bot$, which stand for the negation
of the marked predicates $C$ and $\bot$.  To capture the intended
meaning of these predicates, we introduce rules $X(x)\to\overl\bot(x)$
for $X\in\set{A,B,C,D}$ and a rule $\overl\bot(x)\to\overl\bot(f(x))$
for the unique function symbol $f$ in $\Pex$. The first rules mimick
the usual axiomatisation of $\top$ and ensure that an atom $\overl
\bot(c)$ holds in a Herbrand model of the transposed program whenever
$X(c)$ also holds.  The last rule ensures that $\overl \bot$ holds for
all terms in the Herbrand universe of the transposed program---an
additional requirement that is consistent with the intended meaning of
$\overl \bot$, and critical to the completeness of
transposition in the presence of function symbols.  Finally, a rule
$\overl\bot(z)\wedge C(x) \wedge \overl C(x) \to \bot(z)$ ensures that
the fresh predicate $\overl C$ behaves like the negation of $C$
($\overl\bot(z)$ is added for safety).

The key step of transposition is to 
invert the rules involving the marked predicates by 
shuffling  
marked atoms between head and body while replacing 
their predicate with the corresponding fresh one. 
In this way, rule $B(x) \to C(x) \vee D(x)$ yields
$B(x) \wedge \overl C(x) \to D(x)$, and $C(x) \to \bot(x)$ yields
$\overl \bot(x) \to \overl C(x)$. Additionally, rule
$D(x) \to C(f(x))$ is transposed as $\overl\bot(z) \wedge D(x) \wedge \overl C(f(x)) \to \bot(z)$ to ensure
safety. Finally, transposition does not affect rules containing only Horn predicates, e.g., 
rule $A(x) \to B(x)$ is included unchanged. 

  \begin{definition} \label{def:xiprime} %
  Let $M$ be a marking of a program $\DDP$. For each disjunctive predicate
  $P$ in $\DDP$, let $\overl P$ be a fresh predicate of the same arity.
  The $M$-transposition of $\DDP$
  is the smallest program $\Xi_M(\DDP)$ 
  containing every rule in $\DDP$ involving only Horn predicates
  and all rules 1--6 given next, where
  $\fml$ is the conjunction of all Horn atoms in a rule,
  $\varphi_{\top}$ is the least conjunction of $\overl\bot$-atoms making
  a rule safe and all $P_i$, $Q_j$ are disjunctive:
  \begin{enumerate}[labelindent=0pt,leftmargin=*,noitemsep]
  \item $\fml_\top\land\fml\land\bigwedge_{j=1}^m Q_j(\ve
    t_j)\land\bigwedge_{i=1}^n\overl P_i(\ve s_i)\to\overl Q(\ve t)$ for each 
    rule in $\DDP$
    of the form $\fml\land Q(\ve
    t)\land\bigwedge_{j=1}^m Q_j(\ve t_j)\to\bigvee_{i=1}^n P_i(\ve
    s_i)$  where $Q(\ve
    t)$ is the only marked body atom;
  \item $\overl\bot(x) \wedge
    \fml\land\bigwedge_{j=1}^m Q_j(\ve
    t_j)\land\bigwedge_{i=1}^n\overl P_i(\ve s_i)\to \bot(x)$, where
    $x$ a fresh variable, for each rule in $\DDP$ 
    of the form $\fml\land\bigwedge_{j=1}^m Q_j(\ve
    t_j)\to\bigvee_{i=1}^n P_i(\ve s_i)$, with no marked body atoms
    and no unmarked head atoms;
  \item $
    \fml\land\bigwedge_{j=1}^m Q_j(\ve
    t_j)\land\bigwedge_{i=1}^n\overl P_i(\ve s_i)\to P(\ve s)$ for
    each rule in $\DDP$ 
    of the form $\fml\land\bigwedge_{j=1}^m Q_j(\ve
    t_j)\to P(\ve s)\lor\bigvee_{i=1}^n P_i(\ve s_i)$
    where $P(\ve s)$ is the only unmarked head atom;
  \item $\overl\bot(z) \wedge P(\vec x) \wedge \overl P(\vec x) \to \bot(z)$ for marked predicate $P$;
  \item $P(x_1,\dots,x_n)\to\overl\bot(x_i)$ for 
    each $P$ in $\P$ and $1\le i\le n$;
  \item $\overl\bot(x_1)\land\ldots\land\overl\bot(x_n)\to\overl\bot(f(x_1,\ldots,x_n))$ for
    each $n$-ary function symbol $f$ in $\P$.
    \qedhere
  \end{enumerate}
\end{definition}
Note that rules of type~1 in Def.~\ref{def:xiprime} satisfy
$\set{P_1,\dots,P_n}\subseteq M$ since $Q~{\in}~M$, while for rules of
type~3 we have $\set{Q_1,\dots,Q_m}\cap M=\emptyset$ since
$P\notin M$.

Clearly, $\Pex$ is unsatisfiable when extended with fact
$A(a)$. To see that $\Xi_{M_1}(\Pex) \cup \{A(a)\}$ is also unsatisfiable, note that
$B(a)$ is derived by the unchanged rule $A(x) \to B(x)$. Fact $\overl C(a)$
is derived using $A(x) \to \overl \bot(x)$ and the transposed rule
$\overl \bot(x) \to \overl C(x)$. We derive $D(a)$ using
$B(x) \wedge \overl C(x) \to D(x)$. But then, to derive a contradiction we need to
apply rule $\overl\bot(z) \wedge D(x) \wedge \overl C(f(x)) \to \bot(z)$, which 
is not possible unless we derive $\overl C(f(a))$. For this, we first
use $\overl \bot(x) \to \overl \bot(f(x))$, which ensures that $\overl \bot$ holds for $f(a)$, and then $\overl \bot(x) \to \overl C(x)$.


Transposition yields quadratically many Horn rules. 
The following theorem establishes its correctness.

\begin{restatable}{theorem}{transpositioncorrect} \label{thm:rew-correct-markable} %
  Let $M$ be a
  marking of a program $\DDP$. Then $\Xi_M(\DDP)$ is a polynomial-size Horn
  rewriting of\/ $\DDP$.
\end{restatable}

It follows that
every markable set of non-Horn clauses $\mathcal{N}$ can be polynomially transformed into a 
set of Horn clauses $\mathcal{N}'$ such that $\mathcal{N} \cup \Dat$ and 
$\mathcal{N}' \cup \Dat$ are equisatisfiable for every set of facts $\Dat$.
This result is rather general and
has potential applications in first-order theorem proving, as well as in
knowledge compilation, where Horn clauses are especially relevant \cite{Kcomp,DBLP:journals/ai/Val05}.


%% file: transposition.tex
\section{Markability of DL Ontologies}\label{sec:markability-DLs}

The notion of markability 
is applicable to 
first-order programs and hence
can be seamlessly adapted to ontologies via the
standard translation $\pi$ in Table \ref{tab:RL}.
This, however, would be of
limited value since the Horn programs resulting
from  
transposition  
may not be expressible in Horn-$\mathcal{ALCHIF}$.

Consider any ontology
with an axiom $\exists R.A \sqsubseteq B$
and any marking $M$ involving $R$.
Rule $R(x,y) \wedge A(y) \to B(x)$ stemming from $\pi$ 
would be transposed as $\overl B(x) \wedge A(y) \to \overl R(x,y)$, which
cannot be captured in $\mathcal{ALCHIF}$.\footnote{Capturing such a rule would require a DL that can express products of concepts
\cite{DBLP:conf/dlog/RudolphKH08}.}

\begin{table*}[t]
\begin{footnotesize}
\[
\renewcommand*{\arraystretch}{1.23}
\begin{array}{cl@{\hskip11pt}l@{\hskip15pt}l@{\hskip15pt}l}
& \text{Ontology~} \Oex & \text{Rule translation}~\xi(\Oex) & \text{Transposition}~ \Xi_{\Mex}(\xi(\Oex)) & \text{Horn DL rewriting}~\Psi(\Xi_{\Mex}(\xi(\Oex)))\smallskip\\ \hline  
\raisebox{0pt}[2.5\height][0pt]{$\alpha_1$} &A \sqsubseteq B \sqcup C & A(x) \to B(x) \vee C(x) & A(x) \wedge \overl B(x) \to C(x)              & A \sqcap \overline B \sqsubseteq  C \\
\alpha_2 & B \sqsubseteq \exists R.D & B(x) \to D(f_{R,D}(x)) & \overl D(f_{R,D}(x))\to\overl B(x)           & \exists R_D.\overl D \sqsubseteq \overline B \\
\alpha_3 & \exists R.D \sqsubseteq D & R(x,y) \wedge D(y) \to D(x) & R(x,y)\land\overl D(x)\to\overl D(y)  & \overl D \sqsubseteq \forall R.\overline D \\
         &                    & D(f_{R,D}(x)) \to D(x)      & \overl D(x)\to\overl D(f_{R,D}(x))              & \overl D \sqsubseteq \forall R_D \overline D    \\ 
         &                    & D(f_{R,B}(x)) \to D(x)      & \overl D(x)\to\overl D(f_{R,B}(x))              & \overl D \sqsubseteq \forall R_B \overline D    \\ 
\alpha_4 & C \sqsubseteq \exists R.B    & C(x) \to B(f_{R,B}(x))         & \overl \bot(z) \wedge C(x) \wedge \overl B(f_{R,B}(x)) \to \bot(z)   &  C \sqcap \exists R_B.\overline B \sqsubseteq \bot \\
\alpha_5 & D \sqcap E \sqsubseteq \bot    & D(x) \wedge E(x) \to \bot(x)    & E(x)\land \overl \bot(x) \to \overl D(x)  &  E \sqcap \overl \bot \sqsubseteq \overline D \\
		 &					& 						   & X(x)\to\overl\bot(x),~~X\in\set{A,B,C,D,E}  & X \sqsubseteq \overline \bot \\
                 &                                      &
& R(x_1,x_2)\to\overl\bot(x_i),~~1\le i\le 2  & \top \sqsubseteq\forall R. \overl \bot,~~ \exists R.\top \sqsubseteq \overline \bot \\
		 &					&						   & \overl\bot(x)\to\overl\bot(f_{R,Y}(x)),~~Y \in \{B,D\} & \overl \bot \sqsubseteq \exists R_Y. \overline \bot 
\end{array}
\]
  \caption{Rewriting the example $\mathcal{ELU}$ ontology $\Oex$ into a Horn-$\ALC$ ontology using the marking $\Mex = \{B,D,\bot\}$.}
  \label{tab:ex-transf}
\end{footnotesize}
\end{table*}

To address this limitation we 
introduce an alternative translation $\xi$ from 
DL axioms into rules, which we illustrate using 
the example ontology $\Oex$ in Table \ref{tab:ex-transf}.
The key idea is to encode 
existential restrictions in axioms T3
as unary atoms over functional terms.
For instance, axiom $\alpha_2$ in $\Oex$ would 
yield 
$B(x) \to D(f_{R,D}(x))$, where the ``successor'' relation 
between an instance $b$ of $B$
and some instance 
of $D$ in a Herbrand 
model is encoded as a term $f_{R,D}(b)$, instead
of a binary atom of the form $R(b,g(b))$. 
This encoding has an immediate impact
on markings: by marking $B$ we are only forced to also mark $D$ (rather than both $R$ and $D$). 
In this way,
we will be able to ensure that markings consist of unary predicates only.

To compensate for the lack of binary atoms involving functional terms
in Herbrand models, we introduce new rules when translating 
axioms T2, T4, and T6 using $\xi$.
For instance,  $\xi(\alpha_3)$ 
yields the following rules in addition to $\pi(\alpha_3)$: 
a rule $D(f_{R,D}(x)) \to D(x)$ to ensure that all objects $c$
with an $R$-successor $f_{R,D}(c)$ generated by $\xi(\alpha_2)$ are instances of $D$;
a rule $D(f_{R,B}(x)) \to D(x)$, which makes sure that an
object whose $R$-successor generated by $\xi(\alpha_4)$ is an instance of $D$
is also an instance of $D$. Finally, 
axioms $\alpha_1$ and $\alpha_5$, which involve no binary predicates, are translated as usual.

  \begin{definition} \label{def:little-xi}
  Let $\Ont$ be an ontology. For each 
  concept $\exists R.B$ in an axiom of type T3, let $f_{R,B}$
  be a unary
  function symbol, and $\Phi$ the set of all such symbols. 
  We define $\xi(\Ont)$ as the
  smallest program containing 
  $\pi(\alpha)$ for each axiom $\alpha$ in $\Ont$ of type T1-T2 and T4-T6,
  as well as the following rules:
  \begin{itemize}[labelindent=0pt,leftmargin=*,noitemsep]
    \item $A(x) \to B(f_{R,B}(x))$ for each axiom  T3;
     \item $A(f_{R',Y}(x)) \to C(x)$ for each axiom T2 and  $R'$ and $Y$ s.t. $f_{R',Y} \in \Phi$ and $R'\sqsubseteq^{*} R$.
   \item $A(f_{\mathsf{inv}(R'),Y}(x)) \to C(x)$ for each axiom T4 and  $R'$ and $Y$ s.t. $f_{\mathsf{inv}(R'),Y} \in \Phi$ and $R'\sqsubseteq^{*} R$.
   \item $A(x) \wedge Y(f_{\mathsf{inv}(R'),Y}(x)) \to
     C(f_{\mathsf{inv}(R'),Y}(x))$ for each axiom T2 and $R'$
     and $Y$ s.t.\ $f_{\mathsf{inv}(R'),Y} \in \Phi$ and
     $R'\sqsubseteq^{*} R$.
  \item $A(x) \wedge Y(f_{R',Y}(x)) \to C(f_{R',Y}(x))$ for each axiom T4 and $R'$ and $Y$ s.t.\ $f_{R',Y} \in \Phi$ and $R'\sqsubseteq^{*} R$.
  \item $A(z) \wedge B(f_{R',Y}(z)) \wedge \Inv{R}{z}{x} \wedge B(x)
    \to f_{R',Y}(z) \approx x$ for each ax.\ T6 and $R', Y$ s.t.\
    $f_{R',Y} \in \Phi$ and $R'\sqsubseteq^{*} R$.
  \item $A(f_{\mathsf{inv}(R'),Y}(x))\wedge B(x) \wedge \Inv{R}{f_{\mathsf{inv}(R'),Y}(x)}{y} \wedge B(y)
    \to x \approx y$ for each axiom T6 and $R'$ and $Y$ s.t.\ 
    $f_{\mathsf{inv}(R'),Y} \in \Phi$ and $R'\sqsubseteq^{*} R$.
  \item $A(z) \wedge B(f_{R'_1,Y_1}(z)) \wedge B(f_{R'_2,Y_2}(z)) \to
    f_{R'_1,Y_1}(z) \approx f_{R'_2,Y_2}(z)$ for each axiom T6 and
    $f_{R'_i,Y_i} \in \Phi$ s.t.\ ${R_i'\sqsubseteq^{*} R}$.
  \item $A(f_{\mathsf{inv}(R'_1),Y_1}(x)) \wedge B(x) \wedge
    B(f_{R'_2,Y_2}(f_{\mathsf{inv}(R'_1),Y_1}(x))) \to x \approx
    f_{R'_2,Y_2}(f_{\mathsf{inv}(R'_1),Y_1}(x))$ for each axiom T6 and
    each $R'_i$ and $Y_i$ s.t.\
    $\set{f_{\mathsf{inv}(R'_1),Y_1},f_{R'_2,Y_2}}\subseteq\Phi$ and
    ${R_i'\sqsubseteq^{*} R}$.
    \qedhere
  \end{itemize}
  \end{definition}
Note that, in contrast to the standard translation $\pi$, 
which introduces at most two rules per DL axiom,
$\xi$ can introduce linearly many rules in the size of the role
hierarchy induced by axioms of type T5.

The translation $\xi(\Oex)$ of our example ontology $\Oex$ is
given in the second column of Table \ref{tab:ex-transf}.
Clearly, $\Oex$ is unsatisfiable when extended with $A(a)$ and $E(a)$.
We can check that $\xi(\Oex) \cup \{A(a),E(a)\}$ is also unsatisfiable.
The following theorem establishes the correctness of $\xi$.

\begin{restatable}{theorem}{translationcorrect} \label{thm:xi-correct}
  For every ontology $\Ont$ and dataset $\Dat$ over predicates in
  $\Ont$ we have that $\Ont \cup \Dat$ is satisfiable iff so is
  $\xi(\Ont) \cup \Dat$.
\end{restatable}

This translation has a clear benefit for markability checking: in
contrast to $\pi(\Ont)$, binary predicates in
$\xi(\Ont)$ 
do not belong to any minimal marking. 
In particular, $\Mex = \{B,D,\bot\}$ is the only minimal marking of
$\xi(\Oex)$.
\begin{restatable}{proposition}{xiroleshorn} \label{prop:xi-roles-unmarked}
 \begin{inparaenum}[\it (i)] 
 \item If\/ $\equality$ is Horn in\/ $\xi(\Ont)$ then so are all binary
   predicates in\/ $\xi(\Ont)$.
 \item If $\xi(\Ont)$ is markable, then it has a marking
 containing only unary predicates.
 \end{inparaenum}
\end{restatable}

Thus, we define markability of ontologies in terms of $\xi$
rather than in terms of $\pi$.
We can check that $\pi(\Oex)$ is not markable, whereas
$\xi(\Oex)$ admits the marking $\Mex$.

\begin{definition}
An ontology $\Ont$ is markable if so is $\xi(\Ont)$.
\end{definition}



We conclude this section with the observation that markability 
of an ontology $\Ont$ can be efficiently checked by first computing the program
$\xi(\Ont)$ and then exploiting the 2-SAT encoding
sketched in Section \ref{sec:markability-transposition}.

\section{Rewriting Markable Ontologies} \label{sec:rewriting}

It follows from the correctness of transposition in
Theorem~\ref{thm:rew-correct-markable} and $\xi$
in Theorem \ref{thm:xi-correct} that every $\mathcal{ALCHIF}$ ontology $\Ont$
admitting a marking $M$ has a 
Horn rewriting  of polynomial size given as the program $\Xi_M(\xi(\Ont))$. 
In what follows, we show that this rewriting can be expressed within
Horn-$\mathcal{ALCHIF}$. 

Let us consider the transposition 
of $\xi(\Oex)$ via the marking $\Mex$, which is
given in the third column of Table \ref{tab:ex-transf}.
The transposition of $\alpha_1$ and $\alpha_5$ 
corresponds directly to DL axioms via the standard translation in
Table \ref{tab:RL}. In contrast, the transposition of all 
other axioms leads to rules
that have no direct correspondence in DLs. The following lemma establishes
that the latter rules are restricted to the types T7-T20 specified 
on the left-hand side of Table
\ref{tab:phi}.

\begin{restatable}{lemma}{xiruleform} \label{lem:rew-rule-form} %
  Let $\Ont$ be an ontology and\/ $M$ a minimal marking of\/
  $\xi(\Ont)$.  Then\/ $\Xi_M(\xi(\Ont))$ contains only Horn rules of
  type T1-T2 and T4-T6 in Table\:\ref{tab:RL} as well as 
  type T7-T20 in Table\:\ref{tab:phi}.
\end{restatable}

\begin{table*}[t]
\begin{footnotesize}
  \[
  \renewcommand*{\arraystretch}{1.2}
  \begin{array}{rl@{\qquad}r@{\;}l}
    T7.& \overline\bot(z) \wedge B(x) \wedge R(x,y) \wedge A(y)  \to \bot(z) & B\sqcap\exists R.A&\sqsubseteq\bot\\
    T8.& \overline\bot(z) \wedge A(f_{R,Y}(x)) \wedge B(x) \to \bot(z) & B\sqcap\exists R_Y.A&\sqsubseteq\bot\\
    T9.& \overline\bot(x) \to\overline\bot(f_{R,Y}(x)) & \overline\bot &\sqsubseteq\exists R_Y.\overline\bot\\
    T10.& B(x) \to A(f_{R,Y}(x)) & B&\sqsubseteq\forall R_Y.A \text{~~ if $A\ne\overl\bot$ or $B\ne\overline\bot$}\\
    T11.& B(f_{R,Y}(x))\to A(x) & \exists R_Y.B&\sqsubseteq A\phantom{\overline\bot}\\
    T12.& A(x)\land B(f_{R,Y}(x))\to C(f_{R,Y}(x)) & A \sqcap \exists R_Y.B &\sqsubseteq\forall R_Y.C\phantom{\overline\bot}\\ 
    T13.& \overline\bot(z)\land A(x)\land B(f_{R,Y}(x))\land C(f_{R,Y}(x))\to\bot(z) & A\sqcap\exists R_Y(B\sqcap C)&\sqsubseteq\bot\phantom{\overline\bot}\\
    T14.& B(f_{R,Y}(x))\land C(f_{R,Y}(x))\to A(x) & \exists R_Y(B\sqcap C)&\sqsubseteq A\phantom{\overline\bot} \\
    T15. & A(z) \wedge B(f_{R',Y}(z)) \wedge \Inv{R}{z}{x} \wedge B(x) & R'_Y&\sqsubseteq S_{\set{R'_Y,R}}\text{~~and~~}R\sqsubseteq S_{\set{R'_Y,R}}\text{~~and}   \\
    & \hfill\to f_{R',Y}(z) \approx x & A&\sqsubseteq{\le} 1 S_{\set{R'_Y,R}}.B\\
    T16. & A(f_{R',Y}(x)) \wedge B(x) \wedge \Inv{R}{f_{R',Y}(x)}{y} \wedge B(y) & \tilde R'_Y&\sqsubseteq S_{\set{\tilde R'_Y,R}}\text{~~and~~}R\sqsubseteq S_{\set{\tilde R'_Y,R}}\text{~~and}   \\
    & \hfill\to x \approx y & A&\sqsubseteq{\le} 1 S_{\set{\tilde R'_Y,R}}.B\text{~~and~~}\tilde R'_Y\equiv\mathsf{inv}(R'_Y)\\
    T17. & A(z) \wedge B(f_{R,Y}(z)) \wedge B(f_{R',Z}(z)) & {R}_{Y}&\sqsubseteq S_{\set{{R}_{Y},R'_{Z}}}\text{~~and~~}R'_{Z}\sqsubseteq S_{\set{{R}_{Y},R'_{Z}}}\text{~~and} \\
    & \hfill\to f_{R,Y}(z) \approx f_{R',Z}(z) & A & \sqsubseteq{\le} 1 S_{\set{R_Y,R'_Z}}.B\\
    T18. & A(f_{R,Y}(x)) \wedge B(x) \wedge B(f_{R',Z}(f_{R,Y}(x))) & \tilde R_Y&\sqsubseteq S_{\set{\tilde R_Y,R'_Z}}\text{~~and~~}R'_Z\sqsubseteq S_{\set{\tilde R_Y,R'_Z}}\text{~~and} \\
    & \hfill\to x \approx f_{R',Z}(f_{R,Y}(x)) & A & \sqsubseteq{\le} 1 S_{\set{\tilde R_Y,R'_Z}}.B\text{~~and~~}\tilde R_Y\equiv\mathsf{inv}(R_Y)\\
    T19. & R(x,y)\to\overl\bot(x) & \exists R.\top&\sqsubseteq\overline\bot\\
    T20. & R(x,y)\to\overl\bot(y) & \top&\sqsubseteq\forall R.\overline\bot
  \end{array}
  \]
  \caption{Transformation $\Psi$ from transposed rules to DLs. 
  Role names $\tilde R$ are fresh for every $R$, and $S_{\set{R,R'}}$ 
  for every $\set{R,R'}$.}
  \label{tab:phi}
\end{footnotesize}
\end{table*}

We can now specify a transformation $\Psi$ that allows us to translate rules 
T7-T20 in Table \ref{tab:phi} back into DL axioms.

\begin{definition}
We define $\Psi$ as the transformation mapping 
\begin{inparaenum}[\it (i)]
\item each Horn rule $r$
of types T1-T2 and T4-T6 
in Table~\ref{tab:RL} to the DL axiom $\pi^{-1}(r)$
\item each rule T7-T20 on the left-hand side of Table
\ref{tab:phi} to the DL axioms on the right-hand side.\footnote{For succinctness, axioms resulting from T7, T8, T12, T13, T14, T16 and T18 are not 
given in normal form. }
\end{inparaenum} 
\end{definition}

Intuitively, $\Psi$ works as follows:
\begin{inparaenum}[\it (i)] 
\item Function-free rules are ``rolled up''
as usual into DL axioms (see e.g., T7). 
\item Unary atoms $A(f_{R,Y}(x))$ with $A \neq \overl \bot$ 
involving a functional term
are translated as either existentially or universally 
quantified concepts depending on whether they occur
in the body or in the head (e.g., T10, T11);
in contrast, atoms $\overl \bot(f_{R,Y}(x))$ in
rules $\overl\bot(x)\to\overl\bot(f_{R,Y}(x))$
are translated as $\exists R_Y.\overl \bot$, instead of
$\forall R_Y.\overl \bot$ (see T9). 
\item Rules T15-T18, which involve $\equality$ in the head 
and roles $R'$ and $R$ in the body, 
are rolled back into
axioms of type T6 over the ``union'' of $R$ and $R'$, which is captured
using fresh roles and role inclusions.
\end{inparaenum}

The ontology obtained by applying $\Psi$ to our running example
is given
in the last column of Table \ref{tab:ex-transf}.
Correctness of $\Psi$ and its implications for the computation of Horn 
rewritings are summarised in the following lemma.
\begin{restatable}{lemma}{translationback} \label{thm:back-to-DLs} Let
  $\Ont$ be a markable $\mathcal{ALCHIF}$ ontology and let $M$ be a
  marking of $\Ont$. Then the ontology 
  $\Psi(\Xi_M(\xi(\Ont)))$ is a Horn rewriting of $\Ont$.
\end{restatable}

A closer look at our transformations reveals that
our rewritings do not introduce constructs such as inverse roles and
cardinality restrictions if these were not already present in the input
ontology.  In contrast, fresh role inclusions
may originate from cardinality restrictions in the input
ontology.  As a result,
our approach is language-preserving: if the input $\Ont_1$
is in a DL $\mathcal{L}_1$ between $\mathcal{ALC}$ and
$\mathcal{ALCHI}$, then its rewriting $\Ont_2$ stays in the Horn
fragment of $\mathcal{L}_1$; furthermore, if $\mathcal{L}_1$ is between
$\mathcal{ALCF}$ and $\mathcal{ALCIF}$, then $\Ont_2$ 
may contain fresh role inclusions ($\mathcal{H}$). A notable exception is when $\Ont_1$
is an $\mathcal{ELU}$ ontology, in which case axioms T2 and T3 in $\Ont_1$
may yield axioms of type T4 in $\Ont_2$.  The
following theorem follows from these observations and Lemma
\ref{thm:back-to-DLs}.

\begin{restatable}{theorem}{finalrewritability} \label{thm:final-rewritability}
  Let $\mathcal{L}$ be a DL between $\mathcal{ALC}$ and
  $\mathcal{ALCHI}$. Then every markable $\mathcal{L}$ ontology is
  polynomially rewritable into a Horn-$\mathcal{L}$ ontology. If
  $\mathcal{L}$ is between $\mathcal{ALCF}$ and $\mathcal{ALCHIF}$,
  then every markable $\mathcal{L}$ ontology is polynomially
  rewritable into Horn-$\mathcal{LH}$. Finally, every markable
  $\mathcal{ELU}$ ontology is polynomially rewritable into
  Horn-$\mathcal{ALC}$.
\end{restatable}

\section{Complexity Results}\label{sec:complexity-results}

We next establish the complexity of satisfiability checking
over markable ontologies. 

We first show that satisfiability checking over markable $\mathcal{ELU}$ ontologies
is \textsc{ExpTime}-hard.
This  implies that it is not possible to polynomially rewrite
every markable $\mathcal{ELU}$ ontology into $\mathcal{EL}$. Consequently, 
our rewriting approach is optimal for $\mathcal{ELU}$
in the sense that introducing universal restrictions (or equivalently inverse roles) in the rewriting
is unavoidable.

\begin{restatable}{lemma}{elusatexptimehard} \label{lem:elu-sat-exptime-hard}
Satisfiability checking over markable $\mathcal{ELU}$ ontologies  
is \textsc{ExpTime}-hard.
\end{restatable}

All Horn DLs from $\mathcal{ALC}$ to
$\mathcal{ALCHIF}$ are \textsc{ExpTime}-complete in combined complexity and
\textsc{PTime}-complete
in data complexity \cite{KroetzschRH13}. By Theorem \ref{thm:final-rewritability}, the same
result holds for markable ontologies in DLs from $\mathcal{ALC}$ to
$\mathcal{ALCHIF}$. Finally, Lemma \ref{lem:elu-sat-exptime-hard}
shows that these complexity results also extend to markable $\mathcal{ELU}$ ontologies.

\begin{restatable}{theorem}{complexity}
Let $\mathcal{L}$ be  in-between $\mathcal{ELU}$ and $\mathcal{ALCHIF}$. 
Satisfiability checking over markable $\mathcal{L}$-ontologies is
\textsc{ExpTime}-complete  and \textsc{PTime}-complete
w.r.t.\ data.
\end{restatable}



%% file: related.tex
\section{Related Work}

Horn logics are common target languages for knowledge compilation
\cite{Kcomp}. \citeA{selman1996knowledge} proposed an algorithm
for compiling a set of propositional clauses into a set of Horn clauses s.t.\  their
Horn consequences coincide. This approach was generalised to
FOL by \citeA{DBLP:journals/ai/Val05}, without 
termination guarantees.

\citeA{BienvenuCLW14} showed undecidability of
Datalog rewritability for $\mathcal{ALCF}$ and decidability in
\textsc{NExp\-Time} for 
$\mathcal{SHI}$.
\citeA{grauIJCAI13} and \citeA{KaminskiNCG14RR} proposed
practical techniques for computing Datalog
rewritings of $\mathcal{SHI}$ ontologies based on a two-step process.
First, $\Ont$ 
is rewritten using a resolution calculus
$\Omega$  
into a 
Disjunctive Datalog program $\Omega(\Ont)$ 
of exponential 
size \cite{hms07reasoning}.
Second, $\Omega(\Ont)$ is rewritten into a Datalog program
$\P$. 
For the second step, \citeA{KaminskiNCG14RR} propose the notion of
markability of a Disjunctive Datalog program and show that 
$\P$ can be polynomially computed from $\Omega(\Ont)$ using transposition whenever
$\Omega(\Ont)$ is markable.
In contrast to our work, \citeA{KaminskiNCG14RR}
focus on Datalog as target language for rewriting (rather than Horn DLs).
Furthermore, their Datalog rewritings may be exponential w.r.t.\ the input
ontology and cannot generally be represented in DLs.
 
\citeA{GottlobMMP12} showed tractability in data complexity
of fact entailment for the class of first-order rules
with single-atom bodies, which is sufficient to
capture most DLs in the
$\text{DL-Lite}_{\mathsf{bool}}$
family \cite{Artale09thedl-lite}.

\citeA{Lutz:2012ug} investigated (non-uniform) data complexity of 
query
answering w.r.t.\ \emph{fixed} ontologies. 
They studied the boundary
of \textsc{PTime} and \coNP-hardness and 
established a connection with constraint satisfaction 
problems.
Finally, 
\citeA{DBLP:conf/ijcai/LutzPW11} studied model-theoretic rewritability
of ontologies in a DL $\mathcal{L}_1$ into a fragment $\mathcal{L}_2$
of $\mathcal{L}_1$. These rewritings preserve models rather than just satisfiability, 
which severely restricts the class of
rewritable ontologies; in particular, only ontologies that 
are ``semantically Horn'' can be rewritten. For instance,
$\Ont = \{A \sqsubseteq B \sqcup C\}$, which is rewritable by
our approach, 
is not Horn-rewritable according to \citeA{DBLP:conf/ijcai/LutzPW11}.


%% file: evaluation.tex
\section{Proof of Concept}
 
To assess the practical implications of our results, we have evaluated
whether real-world ontologies are markable (and hence also
polynomially Horn rewritable).  We analysed 120 non-Horn
ontologies 
extracted from the Protege Ontology
Library, 
BioPortal (http://bioportal.bioontology.org/), the corpus by
\citeA{GardinerTH06}, and the EBI linked data platform
(http://www.ebi.ac.uk/rdf/platform). 
To check markability, we have implemented the 2-SAT reduction in
Section~\ref{sec:markability-transposition} and a simple 2-SAT solver.

We found that a total of 32 ontologies were markable and thus
rewritable into a Horn ontology, including some ontologies
commonly used in applications, such as ChEMBL
(see http://www.ebi.ac.uk/rdf/services/chembl/) and
BioPAX Reactome (http://www.ebi.ac.uk/rdf/services/reactome/). When
using $\pi$ as first-order logic translation, we obtained 30 markable
ontologies---a strict subset of the ontologies markable using
$\xi$. However, only 27 ontologies were rewritable to a Horn DL since
in three cases the marking contained a role.
 


%% file: future.tex
\section{Conclusion and Future Work}

We have presented the first practical technique for rewriting non-Horn
ontologies into a Horn DL. Our rewritings are polynomial, and our experiments
suggest that they are applicable to widely-used ontologies. 
We anticipate several directions for future work.
First, we would like to conduct an extensive evaluation to assess whether
the use of our rewritings can significantly speed up
satisfiability checking in practice.
Second, we will investigate relaxations 
of markability that 
would allow us to capture a wider range of ontologies.

%% file: proofs-hornrewritability.tex
\section{Proofs for Section~\ref{sec:horn-rewritability}}

\hornrewundecidable*

\begin{proof}
  We adapt the undecidability proof for datalog-rewritability of
  $\mathcal{ALCF}$ in \cite{BienvenuCLW14}. Given an instance $\Pi$ of
  the undecidable finite rectangle tiling problem,
  \citeauthor{BienvenuCLW14} give an $\mathcal{ALCF}$ ontology
  $\Ont_1$, signature $\Sigma$ and concept name $E$ such that the
  following three conditions are equivalent:
  \begin{itemize}
  \item $\Pi$ admits a tiling
  \item there is a dataset $\Dat$ over $\Sigma$ such that
    $\Ont_1\cup\Dat$ is satisfiable and $\Ont_1\cup\Dat\models E(a)$ for some $a$ in $\Dat$;
  \item there is a dataset $\Dat$ over $\Sigma$ such that
    $\Ont_1\cup\Dat$ is satisfiable and $\Ont_1\cup\Dat\models\exists x.E(x)$.
  \end{itemize}

  Let $S,S'$ be fresh role names and $P_1,P_2,P_3$ fresh concept
  names. Let $\Ont_2$ be an extension of $\Ont_1$ by the following
  axioms.
  \begin{align*}
    \top&\sqsubseteq\exists S.E &
    (\exists S'.\top)\sqcap P_i\sqcap P_j&\sqsubseteq\bot\qquad\text{for~} 1\le i<j\le 3\\
    \exists S'.\top&\sqsubseteq P_1\sqcup P_2\sqcup P_3 &
    (\exists S'.\top)\sqcap P_i\sqcap\exists S'.P_i&\sqsubseteq\bot\qquad\text{for~} 1\le i\le 3
  \end{align*}
  We next show that $\Pi$ admits a tiling if and only if $\Ont_2$ is
  not rewritable to Horn-$\mathcal{L}_2$. First, suppose $\Pi$ admits
  a tiling. Then there is a dataset $\Dat_1$ over $\Sigma$ such that
  $\Ont_1\cup\Dat_1$ is satisfiable and $\Ont_1\cup\Dat_1\models E(a)$
  for some $a$ in $\Dat_1$. Given a connected undirected graph $G$, let
  $\Dat_G=\mset{S'(d,d'),S'(d',d)}{\set{d,d'}\textup{ edge in }G}$ and
  $\Dat_2=\Dat_1\cup\Dat_G\cup\mset{S(d,c)}{d\textup{ occurs in
    }\Dat_G\cup\Dat_1,c\textup{ occurs in }\Dat_1}$. Then $\Ont_2\cup\Dat_2$ is
  consistent if and only if $G$ is 3-colourable. Therefore, since
  3-colourability is \NP{}-complete in data whereas satisfiability
  checking w.r.t. Horn-$\mathcal{ALCHIF}$ ontologies is tractable in
  data, $\Ont_2$ is not rewritable to Horn-$\mathcal{L}_2$ unless
  $\Ptime=\NP$.

  Now suppose $\Pi$ does not admit a tiling. Then $\Ont_2\cup\Dat$ is
  unsatisfiable for every $\Dat$ and hence the ontology
  $\set{\top\sqsubseteq\bot}$ is a Horn-$\ELU$ rewriting of $\Ont_2$.
\end{proof}


%% file: proofs-transposition.tex
\section{Proofs for Section \ref{sec:markability-transposition}}

Reasoning w.r.t.\ programs can be realised by means of the
\emph{hyperresolution calculus}.  In our treatment of hyperresolution
we treat disjunctions of atoms as sets and hence we do not allow for
duplicated atoms in a disjunction.
Let $r=\bigwedge_{i=1}^n \beta_i\to\fml$ be a rule and, for each $1\le
i\le n$, let $\fmm_i$ be a disjunction of atoms $\fmm_i=\fmn_i\lor
\alpha_i$ with $\alpha_i$ a single atom. Let $\sigma$ be an MGU of
each $\beta_i,\alpha_i$. Then the disjunction of atoms
$\fml'=\fml\sigma\lor\fmn_1\lor\dots\lor\fmn_n$
is a \emph{hyperresolvent} of $r$ and $\fmm_1,\dots,\fmm_n$.
Let $\DDP$ be a program, let $\Dat$ be a dataset, and let $\fml$ be a
disjunction of atoms. A (hyperresolution) \emph{derivation} of $\fml$
from $\DDP\cup\Dat$ is a pair $\rho=(T,\lambda)$ where $T$ is a tree,
$\lambda$ a labeling function mapping each node in $T$ to a
disjunction of atoms, and the following properties hold for each $v\in
T$:
\begin{enumerate}
\item $\lambda(v) = \fml$ if $v$ is the root;
\item $\lambda(v) \in\DDP\cup\Dat$ if $v$ is a leaf; and
\item if\/ $v$ has children $w_1, \ldots, w_n$, then $\lambda(v)$ is a
  hyperresolvent of a rule $r \in \DDP$ and $\lambda(w_1), \ldots,
  \lambda(w_n)$.
\end{enumerate}
We write $\DDP\cup\Dat\vdash\fml$ to denote that $\fml$ has a
derivation from $\DDP\cup\Dat$. Hyperresolution is sound and complete
in the following sense: $\DDP\cup\Dat$ is unsatisfiable iff
$\DDP\cup\Dat\vdash\square$. Furthermore, if $\DDP\cup\Dat$ is
satisfiable then $\DDP\cup\Dat\vdash \alpha$ iff
$\P\cup\Dat\models\alpha$ for every atom
$\alpha$. 




  


Hyperresolution derivations satisfy the following property.
  
\begin{proposition}\label{prop:hyperresolution}
  Let $\DDP$ be a program, $\Dat$ a dataset, and $\rho$ a derivation
  from $\DDP\cup\Dat$. Then every node in $\rho$ is labeled by either
  a single Horn atom or a (possibly empty) disjunction of disjunctive
  atoms.
\end{proposition}

\begin{proof}
  The claim follows by a straightforward induction on $\rho$.
\end{proof}

We call a node in a derivation Horn (resp.\ disjunctive) if it is
labeled by a Horn atom (resp.\ a disjunction of disjunctive atoms).



\begin{proposition} \label{prop:bot-bar-axiomatisation}
  Let $\DDP$ be a program, $M$ a marking of\/ $\DDP$, and\/ $\Dat$ a
  dataset over the predicates in $\DDP$. Then\/
  $\Xi_M(\DDP)\cup\Dat\models\overl\bot(s)$ for every ground term\/
  $s$ over the signature of\/ $\DDP\cup\Dat$.
\end{proposition}

\begin{proof}
  The claim is a straightforward consequence of the axiomatisation of
  $\overl\bot$ in $\Xi_M(\DDP)$.
\end{proof}

\transpositioncorrect*

\begin{proof}
  We proceed in two steps, which together imply the theorem.  We fix
  an arbitrary markable program $\DDP$, a marking $M$ of $\DDP$, and a
  dataset $\Dat$. 
  W.l.o.g. we assume that $\Dat$ only contains predicates in $\P$.
  \begin{enumerate}
  \item We show that $\DDP\cup\Dat\models\square$ implies
    $\Xi_M(\DDP)\cup\Dat\models\square$. For this, we consider a
    derivation $\rho$ of $\square$ from $\DDP\cup\Dat$ and show that
    for every disjunctive atom $Q(\ve s)$ in the label of a node in
    $\rho$, we have $\Xi_M(\DDP)\cup\Dat\models\overl Q(\ve s)$ if
    $Q\in M$ and otherwise $\Xi_M(\DDP)\cup\Dat\models Q(\ve s)$. This
    claim, in turn, is shown by first showing a more general statement
    and then instantiating it with~$\rho$.
  \item We show that $\Xi_M(\DDP)\cup\Dat\models\square$ implies
    $\DDP\cup\Dat\models\square$. Again, we first show a general claim
    that holds for any derivation from $\Xi_M(\DDP)\cup\Dat$ and then
    instantiate the claim with a derivation of\, $\square$.
  \end{enumerate}
  In both steps we use that $\DDP$ and $\Xi_M(\DDP)$ entail the same
  facts over Horn predicates for every dataset.  We now detail the two
  steps formally.

  \medskip
  \noindent
  \textbf{Step 1. } Suppose $\DDP\cup\Dat\models\square$. We show
  $\Xi_M(\DDP)\cup\Dat\models\square$.
  We begin by showing the following claim.

  \smallskip
  \noindent
  Claim $(\diamondsuit)$.  Let $\fml=Q_1(\ve s_1)\lor\dots\lor Q_n(\ve
  s_n)$ be a non-empty
  disjunction of facts satisfying the following properties:
  \begin{inparaenum}[\it(i)]
  \item $\Xi_M(\P)\cup\Dat\models\overl Q_i(\ve
  s_i)$ for each 
  $Q_i\in M$.
  \item $\fml$ is derivable from $\P\cup\Dat$.
  \end{inparaenum}  
  Then, for each derivation $\rho$ of $\fml$ from $\P\cup\Dat$ and
  each atom $R(\ve t)$ with $R$ disjunctive in the label of a core
  node in $\rho$ we have $\Xi_M(\P)\cup\Dat\models\overl R(\ve t)$ if
  $R\in M$ and $\Xi_M(\P)\cup\Dat\models R(\ve t)$ otherwise.

  \smallskip

  We show the claim by induction on $\rho=(T,\lambda)$. W.l.o.g., the
  root $v$ of $T$ has a disjunctive predicate in its label (otherwise,
  the claim is vacuous since the core of $\rho$ contains no
  disjunctive nodes).

  For the base case, suppose $v$ has no children labeled with
  disjunctive predicates. We then distinguish two subcases:
  \begin{itemize}
  \item $\fml\in\Dat$. Then $\fml$ is a fact, i.e., $\fml=Q(\ve a)$
    for some $Q$ and $\ve a$. If $Q\in M$, the claim is immediate by
    assumption \emph{(i)}. If $Q\ne M$, the claim follows as
    $\Dat\models Q(\ve a)$.
  \item $\fml$ is obtained by a rule $\fmm\to\fml'\in\DDP$ where
    $\fmm$ is a conjunction of Horn atoms and, for some $\sigma$,
    $\fml=\fml'\sigma$ and $\DDP\cup\Dat\models\fmm\sigma$.  If
    $\set{Q_1,\dots,Q_n}\subseteq M$, the claim is immediate by
    assumption \emph{(i)}, so let us assume w.l.o.g.\ that $Q_1\notin
    M$. By the definition of a marking, we then have
    $\set{Q_2,\dots,Q_n}\subseteq M$, and hence it suffices to show
    $\Xi_M(\P)\cup\Dat\models Q_1(\ve s_1)$. This follows since
    $
    \fmm\land\bigwedge_{i=2}^n \overl Q_i(\veprime s_i)\to
    Q_1(\veprime s_1)\in\Xi_M(\P)$ (where $\veprime s_i\sigma=\ve s_i$),
    $\Xi_M(\P)\cup\Dat\models\bigwedge_{i=2}^n\overl Q_i(\ve s_i)$ by
    assumption \emph{(i)}, $\Xi_M(\P)\cup\Dat\models\fmm\sigma$ since
    $\Xi_M(\DDP) \cup \Dat$ and $\P \cup \Dat$ entail the same Horn
    atoms. 
  \end{itemize}
 
  For the inductive step, suppose $v$ has children $w_1,\dots,w_m$ in
  $T$ that are labeled with disjunctive predicates. W.l.o.g., there is
  a rule $r=\fmm\land\bigwedge_{i=1}^m R_i(\veprime
  t_i)\to\bigvee_{j=1}^k Q_j(\veprime s_j)$ in $\DDP$ (with $\fmm$ a
  conjunction of Horn atoms, $0\le k\le n$, and all $R_i$ disjunctive
  in $\P$) such that $\lambda(v)$ is obtained by a hyperresolution
  step using $r$ from $\fmm\sigma$ and
  $\lambda(w_1),\dots,\lambda(w_m)$ where $\sigma$ is a substitution
  mapping every atom $R_i(\ve t_i)$ to a disjunct in
  $\lambda(w_i)$. In particular, we have $\veprime s_j\sigma=\ve s_j$,
  $R_i(\veprime t_i\sigma)\in\lambda(w_i)$, and
  $\DDP\cup\Dat\models\fmm\sigma$.  We distinguish three
  cases:
  \begin{itemize}
  \item $\set{Q_1,\dots,Q_k}\subseteq M$ and $\set{R_1,\dots,R_m}\cap
    M=\emptyset$. Then, for every $i\in[1,m]$, every marked atom in
    $\lambda(w_i)$ also occurs in $\lambda(v)$; furthermore,
    every unmarked
    atom in $\lambda(v)$ occurs in $\lambda(w_i)$ for some
    $i\in[1,m]$. By the latter statement, it suffices to show the claim for the
    subderivations rooted at $w_1,\dots,w_m$.

    Let $i\in[1,m]$. By the fact that every marked atom in
    $\lambda(w_i)$ also occurs in $\lambda(v)$ and assumption
    \emph{(i)}, we have $\Xi_M(\P)\cup\Dat\models\overl S(\ve u)$ for
    every marked disjunct $S(\ve u)$ in $\lambda(w_i)$. Then, we can
    apply the inductive hypothesis to the subderivation rooted at
    $w_i$ and the claim follows.
  \item $\set{Q_1,\dots,Q_k}\subseteq M$, $R_1\in M$, and
    $\set{R_2,\dots,R_m}\cap M=\emptyset$ (note that $R_1\in M$
    implies $\set{R_2,\dots,R_m}\cap M=\emptyset$ since $M$ is a
    marking). Then (a) for every $i\in[1,m]$, every marked atom in
    $\lambda(w_i)$ except for possibly $R_1(\veprime t_1\sigma)$ in
    $\lambda(w_1)$ also occurs in $\lambda(v)$, and (b) every unmarked
    atom in $\lambda(v)$ occurs in $\lambda(w_i)$ for some
    $i\in[1,m]$. Also, we have (c)
    $\fml_\top\land\fmm\land\bigwedge_{i=2}^m R_i(\veprime
    t_i)\land\bigwedge_{j=1}^k\overl Q_j(\veprime s_j)\to\overl
    R_1(\veprime t_1)\in\Xi_M(\P)$.  As in the preceding case, by (b),
    it suffices to show the claim for the subderivations rooted at
    $w_1,\dots,w_m$.  For $w_2,\dots,w_n$, we proceed as follows. Let
    $i\in[2,m]$. By (a) and assumption \emph{(i)}, we have
    $\Xi_M(\P)\cup\Dat\models\overl S(\ve u)$ for every marked
    disjunct $S(\ve u)$ in $\lambda(w_i)$. Thus, we can apply the
    inductive hypothesis to the subderivation rooted at $w_i$. In
    particular, we obtain $\Xi_M(\DDP)\cup\Dat\models R_i(\veprime
    t_i\sigma)$.  In the case of $w_1$, we need to show
    $\Xi_M(\DDP)\cup\Dat\models\overl R_1(\veprime t_1\sigma)$ in
    order to apply the inductive hypothesis. This follows by (c) and
    assumption \emph{(i)} since
    $\Xi_M(\DDP)\cup\Dat\models\fmm\sigma$, $\set{Q_1(\veprime
      s_1\sigma),\dots,Q_k(\veprime s_k\sigma)}\subseteq\lambda(v)$,
    $\Xi_M(\DDP)\cup\Dat\models R_i(\veprime t_i\sigma)$ for
    $i\in[2,m]$, and
    $\Xi_M(\DDP)\cup\Dat\models\fml_\top\sigma$.
  \item $Q_1\notin M$, $\set{Q_2,\dots,Q_k}\subseteq M$, and
    $\set{R_1,\dots,R_m}\cap M=\emptyset$ (note that $Q_1\notin M$
    implies $\set{Q_2,\dots,Q_k}\subseteq M$ and
    $\set{R_1,\dots,R_m}\cap M=\emptyset$). Then (a) for every
    $i\in[1,m]$, every marked atom in $\lambda(w_i)$ also occurs in
    $\lambda(v)$, and (b) every unmarked atom in $\lambda(v)$ except
    for possibly $Q_1(\ve s_1)$ (but including
    $Q_2(s_2),\dots,Q_m(s_m)$) occurs in $\lambda(w_i)$ for some
    $i\in[1,m]$. By (b), it suffices to show the main claim for the
    subderivations rooted at $w_1,\dots,w_m$ and also that
    $\Xi_M(\P)\cup\Dat\models Q_1(\ve s_1)$.  Let $i\in[1,m]$. The
    main claim for the subderivations follows from (a) and assumption
    \emph{(i)}, which imply that $\Xi_M(\P)\cup\Dat\models\overl S(\ve
    u)$ for every marked disjunct $S(\ve u)$ in $\lambda(w_i)$; as a
    result, we can apply the inductive hypothesis to the subderivation
    rooted at $w_i$.  Finally, note that
    $
    \fmm\land\bigwedge_{i=1}^m R_i(\veprime
    t_i)\land\bigwedge_{j=2}^k\overl Q_j(\veprime s_j)\to Q_1(\veprime
    s_1)\in\Xi_M(\P)$ (since $r\in\P$).  Then,
    $\Xi_M(\P)\cup\Dat\models Q_1(\ve s_1)$ follows from
    $\Xi_M(\P)\cup\Dat\models \fmm\sigma$, the inductive hypothesis
    (which implies $\Xi_M(\P)\cup\Dat\models R_i(\veprime
    t_i\sigma)$), and the assumption \emph{(i)} (which implies
    $\Xi_M(\P)\cup\Dat\models \overl Q_j(\veprime s_j\sigma)$).
  \end{itemize}
  We next instantiate $(\diamondsuit)$ to show the claim in Step 1.
  Let $\varphi = \bot(s)$. We have assumed in Step 1 that $\DDP \cup
  \Dat \models\square$\, so $\bot(s)$ is derivable from $\DDP \cup
  \Dat$ for some $s$ (as $\square$ can only be derived by the rule
  $\bot(x)\to\square$), and hence condition \emph{(ii)} in
  $(\diamondsuit)$ holds.
  Furthermore, if $\bot\in M$, we have $\Xi_M(\P)\cup\Dat\models\overl
  \bot(s)$; 
  hence, condition \emph{(i)} in $(\diamondsuit)$ also holds.
  
  Now, let $\rho = (T,\lambda)$ be a derivation of $\bot(s)$ from
  $\DDP\cup\Dat$.  We exploit $(\diamondsuit)$ applied to $\rho$ to
  show that $\Xi_M(\DDP)\cup\Dat\models\bot(s)$. We distinguish two
  cases:
  \begin{itemize}
    \item $\bot\notin M$. Since $\bot(s)$ labels the root of $\rho$ we can
    apply $(\diamondsuit)$ to obtain 
  $\Xi_M(\P)\cup\Dat\models\bot(s)$; 
  the claim follows.
\item $\bot\in M$. Then there is a core node $v$ in $\rho$ such that:
  $\lambda(v)$ contains only marked atoms and $v$ has no successor $w$
  in $T$ such that all atoms in $\lambda(w)$ are marked.  We
  distinguish two cases.

  If $\lambda(v)\in\Dat$, then $\lambda(v)=Q(\ve b)$ for some $Q$ and
  $\ve b$. Moreover, by $(\diamondsuit)$, we have
  $\Xi_M(\P)\cup\Dat\models\overl Q(\ve b)$. The claim follows since
  $\overline\bot(z)\land Q(\ve x)\land\overl Q(\ve x)\to\bot(z)\in\Xi_M(\P)$
  and $\Xi_M(\P)\cup\Dat\models\overl\bot(s)$.

  If $\lambda(v)\notin\Dat$, then $v$ has successors $v_1,\dots,v_n$
  ($n\ge 0$) in $T$ such that $\lambda(v)$ is a hyperresolvent of
  $\lambda(v_1),\dots,\lambda(v_n)$ and a rule in $\P$ of the form
  $\bigwedge_{i=1}^n Q_i(\ve s_i)\to\bigvee_{j=1}^m R_j(\ve t_j)$,
  where the atoms $Q_i(\ve s_i)$ are resolved with
  $\lambda(v_i)$. Since, $\lambda(v)$ contains only marked atoms but
  $\lambda(v_1),\dots,\lambda(v_n)$ all contain Horn or unmarked
  atoms, all $Q_i$ must be Horn or unmarked and all $R_j$ must be
  marked. Hence, $\Xi_M(\P)$ contains a rule
  $r=\overl\bot(x)\land(\bigwedge_{i=1}^k Q_i(\ve
  s_i))\land(\bigwedge_{l=k+1}^n Q_l(\ve
  s_l))\land\bigwedge_{j=1}^m\overl R_j(\ve t_j)\to\bot(x)$ where,
  w.l.o.g., $Q_1,\dots,Q_k$ are Horn and $Q_{k+1},\dots,Q_n$ are
  disjunctive and unmarked. Let $\sigma$ be the substitution used in
  the hyperresolution step deriving $\lambda(v)$. By $(\diamondsuit)$,
  we then have $\Xi_M(\P)\cup\Dat\models Q_l(\ve s_l\sigma)$ for every
  $l\in[k+1,n]$ and $\Xi_M(\P)\cup\Dat\models\overl R_j(\ve
  t_j\sigma)$ for every $j\in[1,m]$. Moreover, we have
  $\lambda(v_i)=Q_i(\ve s_i\sigma)$ and hence
  $\Xi_M(\P)\cup\Dat\models Q_i(\ve s_i\sigma)$ for every
  $i\in[1,k]$. Finally, we have
  $\Xi_M(\P)\cup\Dat\models\overl\bot(s)$. The claim follows
  with~$r$.
  \end{itemize}

  \medskip
  \noindent
  \textbf{Step 2. } Let $\Xi_M(\DDP)\cup\Dat\models\square$.  Then
  there is a derivation $\rho$ of $\bot(s)$ for some $s$ from
  $\Xi_M(\DDP)\cup\Dat$.
  The fact that $\DDP\cup\Dat\models\bot(s)$ follows directly from
  Statement 1 in Claim $(\clubsuit)$, which we show next.
 
  \smallskip
  \noindent
  Claim $(\clubsuit)$. Let $\rho$ be a derivation from
  $\Xi_M(\DDP)\cup\Dat$, and let $v$ be the root of $\rho$. Then:
  \begin{enumerate}
  \item If $\lambda(v)=Q(\ve t)$, then
    $\DDP\cup\Dat\models Q(\ve t)$.
  \item If $\lambda(v)=\overl Q(\ve t)$, then
    $\DDP\cup\Dat\models\neg Q(\ve t)$.
  \end{enumerate}

  We show the two claims by simultaneous induction on~$\rho$.
  For the base case, suppose $v$ is the only node in $\rho$. We
  distinguish two cases:
  \begin{itemize}
  \item $\lambda(v)\in\Dat$. Then $\Dat\models\lambda(v)$ and the
    claim is immediate.
  \item $\lambda(v)=Q(\ve t)$ where $Q$ is Horn in $\DDP$ and
    $r=(\to Q(\ve t))\in\Xi_M(\DDP)$. Then $r\in\DDP$ and the claim
    follows.
  \end{itemize}
  For the inductive step, suppose $v$ has children $v_1,\dots,v_n$
  and, $\lambda(v)$ is a hyperresolvent of
  $\lambda(v_1),\dots,\lambda(v_n)$ and a rule $r\in\Xi_M(\DDP)$. We
  distinguish five cases:
  \begin{itemize}
  \item $r$ contains no disjunctive predicates, in which case
    the claim follows since $\DDP \cup \Dat$ and $\Xi_M(\DDP) \cup \Dat$ 
    entail the same facts over a Horn predicate.
  \item $r=\overl\bot(z)\land P(\ve x)\land\overl P(\ve x)\to\bot(z)$. Then
    $\lambda(v)=\bot(s)$ for some $s$. Since, by the inductive
    hypothesis, $\P\cup\Dat\models P(\ve t)\land\overl P(\ve t)$ for some
    $\ve t$, $\P\cup\Dat$ is inconsistent, and hence
    $\P\cup\Dat\models\bot(s)$.
  \item $r=\fml_\top\land\fml\land\bigwedge_{j=1}^m Q_j(\ve
    t_j)\land\bigwedge_{i=1}^n\overl P_i(\ve s_i)\to\overl Q(\ve t)$
    where $\fml$ is the conjunction of all Horn atoms in $r$ and
    $r'=\fml\land Q(\ve t)\land\bigwedge_{j=1}^m Q_j(\ve
    t_j)\to\bigvee_{i=1}^n P_i(\ve s_i)\in\DDP$. Then
    $\lambda(v)=\overl Q(\ve s)$ for some $\ve s$. For some $\sigma$,
    we have $\DDP\cup\Dat\models\fml\sigma$, $\ve t\sigma=\ve s$ and,
    for each $i,j$, $\Xi_M(\DDP)\cup\Dat\models\overl P_i(\ve
    s_i\sigma)$ and $\Xi_M(\DDP)\cup\Dat\models Q_j(\ve
    t_j\sigma)$. Then, by the inductive hypothesis,
    $\DDP\cup\Dat\models\neg P_i(\ve s_i\sigma)$ and
    $\DDP\cup\Dat\models Q_j(\ve t_j\sigma)$. With $r'$, we obtain
    $\DDP\cup\Dat\models\neg Q(\ve s)$.
  \item $r=\overl\bot(x)\land\fml\land\bigwedge_{j=1}^m Q_j(\ve
    t_j)\land\bigwedge_{i=1}^n\overl P_i(\ve s_i)\to\bot(x)$ where
    $\fml$ is the conjunction of all Horn atoms in $r$ and
    $r'=\fml\land\bigwedge_{j=1}^m Q_j(\ve t_j)\to\bigvee_{i=1}^n
    P_i(\ve s_i)\in\DDP$. Then $\lambda(v)=\bot(s)$ for some $s$. For
    some $\sigma$, we then have $\DDP\cup\Dat\models\fml\sigma$ and,
    for each $i,j$, $\Xi_M(\DDP)\cup\Dat\models\overl P_i(\ve
    s_i\sigma)$ and $\Xi_M(\DDP)\cup\Dat\models Q_j(\ve
    t_j\sigma)$. Then, by the inductive hypothesis,
    $\DDP\cup\Dat\models\neg P_i(\ve s_i\sigma)$ and
    $\DDP\cup\Dat\models Q_j(\ve t_j\sigma)$. With $r'$, we obtain
    that $\DDP\cup\Dat$ is inconsistent and hence
    $\DDP\cup\Dat\models\bot(s)$.
  \item $r=
    \fml\land\bigwedge_{j=1}^m Q_j(\ve
    t_j)\land\bigwedge_{i=1}^n\overl P_i(\ve s_i)\to {P'}(\ve s)$
    where $\fml$ is the conjunction of all Horn atoms in $r$ and
    $r'=\fml\land\bigwedge_{j=1}^m Q_j(\ve t_j)\to P'(\ve
    s)\lor\bigvee_{i=1}^n P_i(\ve s_i)$ in $\DDP$. Then
    $\lambda(v)={P'}(\ve t)$ for some $\ve t$. For
    some $\sigma$ we then have $\DDP\cup\Dat\models\fml\sigma$, $\ve
    s\sigma=\ve t$ and, for each $i,j$,
    $\Xi_M(\DDP)\cup\Dat\models\overl P_i(\ve s_i\sigma)$ and
    $\Xi_M(\DDP)\cup\Dat\models Q_j(\ve t_j\sigma)$. Then, by
    the inductive hypothesis, $\DDP\cup\Dat\models\neg P_i(\ve
    s_i\sigma)$ and $\DDP\cup\Dat\models Q_j(\ve
    t_j\sigma)$. With $r'$, we obtain
    $\DDP\cup\Dat\models P'(\ve t)$.\qedhere
  \end{itemize}
\end{proof}


%% file: proofs-xi.tex
\section{Proofs for Section \ref{sec:markability-DLs}}
\label{sec:proofs-xi}

\translationcorrect*

\begin{proof}
  For the direction from left to right, suppose $\modI$ is a model of
  $\Ont\cup\Dat$. We define the interpretation $\modJ$ such that
  \begin{itemize}
  \item the domain of $\modJ$ extends the domain of $\modI$ by one
    additional individual $u$;
  \item $\modJ$ coincides with $\modI$ on every concept name, role
    name and constant in
    $\Ont\cup\Dat$, and\, ${\approx^\modJ}={\approx^\modI}\cup\set{(u,u)}$;
  \item $f_{R,A}^\modJ(v)\in\mset{w\in A^\modI}{(v,w)\in R^\modI}$ if
    the set $\mset{w\in A^\modI}{(v,w)\in R^\modI}$ is nonempty and
    otherwise $f_{R,A}^\modJ(v)=u$
    (if $R=S^-$ for a role name $S$, we write $R^\modI$ for
    $(S^\modI)^{-1}$).
  \end{itemize}
  We show that $\modJ$ is a model of $\xi(\Ont)\cup\Dat$. Clearly,
  $\modJ$ satisfies $\Dat$ and every rule in $\xi(\Ont)$ of type T1-T2
  and T4-T6, so it suffices to show that $\modJ$ satisfies the rules
  introduced by $\xi$. So, let $r=\xi(\alpha)\setminus\pi(\alpha)$ for
  some $\alpha\in\Ont$. We distinguish the following cases:
  \begin{enumerate}
  \item $r=A(x)\to B(f_{R,B}(x))$ and $\alpha=A\sqsubseteq\exists
    R.B$. Let $v\in A^\modJ$. It suffices to show that
    $f_{R,B}^\modJ(v)\in B^\modJ$. Since $\modI$ satisfies $\alpha$,
    $v$ has an $R^\modI$-successor that is in $B^\modI$, and hence
    $f^\modJ_{R,B}(v)\in B^\modJ=B^\modI$.
  \item $r=A(f_{R',Y}(x))\to C(x)$, $\alpha=\exists R.A\sqsubseteq C$,
    and $R'\sqsubseteq^* R$. Let $f_{R',Y}^\modJ(v)\in A^\modJ$. It
    suffices to show $v\in C^\modJ$. By construction, we have
    $f_{R',Y}^\modJ(v)\ne u$ and hence $f_{R',Y}^\modJ(v)\in\mset{w\in
      A^\modI}{(v,w)\in {R'}^\modI}$. Since $R'\sqsubseteq^* R$, it
    follows that $f_{R',Y}^\modJ(v)\in\mset{w\in A^\modI}{(v,w)\in
      R^\modI}$, i.e., $v\in(\exists R.A)^\modI$. Since $\modI$
    satisfies $\alpha$, we conclude $v\in C^\modI=C^\modJ$.
  \item $r=A(x)\land Y(f_{\mathsf{inv}(R'),Y}(x))\to
    C(f_{\mathsf{inv}(R'),Y}(x))$, $\alpha=\exists R.A\sqsubseteq C$
    and $R'\sqsubseteq^* R$. Let $v\in A^\modJ$ and
    $f_{\mathrm{inv}(R'),Y}^\modJ(v)\in Y^\modJ$. It suffices to show
    $f_{\mathsf{inv}(R'),Y}^\modJ(v)\in C^\modJ$. By construction, we have
    $f_{\mathsf{inv}(R'),Y}^\modJ(v)\ne u$ and hence
    $f_{\mathsf{inv}(R'),Y}^\modJ(v)\in\mset{w\in
      Y^\modI}{(w,v)\in{R'}^\modI}$. Since $R'\sqsubseteq^* R$, it
    follows that $f_{\mathsf{inv}(R'),Y}^\modJ(v)\in\mset{w\in
      Y^\modI}{(w,v)\in R^\modI}$. Since $v\in A^\modJ=A^\modI$, we
    have $f_{\mathsf{inv}(R'),Y}^\modJ(v)\in(\exists R.A)^\modI$. Since
    $\modI$ satisfies $\alpha$, we conclude
    $f_{\mathsf{inv}(R'),Y}^\modJ(v)\in C^\modI=C^\modJ$.
  \item $r=A(f_{\mathsf{inv}(R'),B}(x))\to C(x)$,
    $\alpha=A\sqsubseteq\forall R.C$, and $R'\sqsubseteq^* R$. The
    claim follows similarly to Case 2.
  \item $r=A(x)\land Y(f_{R',Y}(x))\to C(f_{R',Y}(x))$,
    $\alpha=A\sqsubseteq\forall R.C$, and $R'\sqsubseteq^* R$. The
    claim follows similarly to Case 3.
  \item $r=A(z)\land B(f_{R',Y}(z))\land\Inv{R}{z}{x}\land B(x)\to
    f_{R',Y}(z)\equality x$, $\alpha= A\sqsubseteq{\le}1\, R.B$, and
    $R'\sqsubseteq^* R$. Let $v\in A^\modJ$, $f_{R',Y}^\modJ(v)\in
    B^\modJ$, $(v,w)\in R^\modJ$ and $w\in B^\modJ$. It suffices to
    show $f_{R',Y}^\modJ(v)\equality_\modJ w$. By construction, we
    have
    $f_{R',Y}^\modJ(v)\in\mset{w'}{(v,w')\in{R'}^\modI}\subseteq\mset{w'}{(v,w')\in
      R^\modI}$. The claim follows since $A^\modJ=A^\modI$,
    $B^\modJ=B^\modI$, $R^\modJ=R^\modI$ and $\modI$ satisfies~$\alpha$.
  \item $r=A(f_{\mathsf{inv}(R'),Y}(x))\wedge B(x) \wedge
    \Inv{R}{f_{\mathsf{inv}(R'),Y}(x)}{y} \wedge B(y) \to x \equality
    y$, $\alpha= A\sqsubseteq{\le}1\, R.B$, and $R'\sqsubseteq^*
    R$. Let $f_{\mathsf{inv}(R'),Y}^\modJ(v)\in A^\modJ$, $v\in
    B^\modJ$, $(f_{\mathsf{inv}(R'),Y}^\modJ(v),w)\in R^\modJ$, and
    $w\in B^\modJ$. It suffices to show $v\equality_\modJ w$. By
    construction, we have
    $f_{\inv{R'},Y}^\modJ(v)\in\mset{w'}{(v,w')\in\mathsf{inv}(R')^\modI}\subseteq\mset{w'}{(w',v)\in
      R^\modI}$. The claim follows since $A^\modJ=A^\modI$,
    $B^\modJ=B^\modI$, $R^\modJ=R^\modI$, and $\modI$ satisfies $\alpha$.
  \item $r=A(z)\land B(f_{R'_1,Y_1}(z))\land B(f_{R'_2,Y_2}(z))\to
    f_{R'_1,Y_1}(z)\equality f_{R'_2,Y_2}(z)$, $\alpha=
    A\sqsubseteq{\le}1\, R.B$, and $R'_i\sqsubseteq^* R$. The claim
    follows similarly to Case 6.
  \item $r=A(f_{\mathsf{inv}(R'_1),Y_1}(x)) \wedge B(x) \wedge
    B(f_{R'_2,Y_2}(f_{\mathsf{inv}(R'_1),Y_1}(x))) \to x \approx
    f_{R'_2,Y_2}(f_{\mathsf{inv}(R'_1),Y_1}(x))$, $\alpha=
    A\sqsubseteq{\le}1\, R.B$, and $R'_i\sqsubseteq^* R$. The claim
    follows similarly to Case 7.
  \end{enumerate}

  For the direction from right to left, suppose $\modJ$ is a minimal
  Herbrand model of $\xi(\Ont)\cup\Dat$.
  We define the interpretation $\modI$ such that
  \begin{itemize}
  \item $\modI$ coincides with $\modJ$ on its domain as well as on
    every concept name and every constant in $\Ont\cup\Dat$;
  \item
    $R^\modI=R^\modJ\cup\mset{(v,f_{R',Y}^\modJ(v))}{f_{R',Y}\in\Phi,v\in\Delta^\modJ,
      f_{R',Y}^\modJ(v)\in Y^\modJ,R'\sqsubseteq^*
      R}$\\
    \phantom{a}\hskip39pt$\cup\:\mset{(f_{\mathsf{inv}(R'),Y}^\modJ(v),v)}{f_{\mathsf{inv}(R'),Y}\in\Phi,v\in\Delta^\modJ,
      f_{\mathsf{inv}(R'),Y}^\modJ(v)\in Y^\modJ,R'\sqsubseteq^* R}$.
  \end{itemize}
  We show that $\modI$ is a model of $\Ont\cup\Dat$. Clearly, $\modI$
  satisfies $\Dat$ and every axiom in $\Ont$ of type T1, so it
  suffices to show that $\modI$ satisfies axioms of type T2-T6, which
  we do next.
  \begin{itemize}
  \item Let $\exists R.A\sqsubseteq C\in\Ont$. W.l.o.g., let
    $v\in(\exists R.A)^\modI\setminus(\exists R.A)^\modJ$ (if
    $v\in(\exists R.A)^\modJ$ the claim is immediate since
    $\pi(\exists R.A\sqsubseteq C)\in\xi(\Ont)$). It suffices to show
    $v\in C^\modI$. By construction of $R^\modI$, there exists some
    $R'\sqsubseteq^* R$ and $Y$ such that either $f_{R',Y}\in\Phi$ and
    $f_{R',Y}^\modJ(v)\in Y^\modJ\cap A^\modJ$ or
    $f_{\mathsf{inv}(R'),Y}\in\Phi$ and there is some
    $w\in\Delta^\modJ$ such that $v=f_{\mathsf{inv}(R'),Y}^\modJ(w)\in
    Y^\modJ$ and $w\in A^\modI=A^\modJ$. In the former case,
    $v\in C^\modJ=C^\modI$ follows since $A(f_{R',Y}(x))\to
    C(x)\in\xi(\Ont)$. In the latter case, $v\in C^\modI$
    follows since $A(x)\land Y(f_{\mathsf{inv}(R'),Y}(x))\to
    C(f_{\mathsf{inv}(R'),Y}(x))\in\xi(\Ont)$.
  \item Let $A\sqsubseteq\exists R.B\in\Ont$ and let $v\in
    A^\modI$. It suffices to show $v\in(\exists R.B)^\modI$. Since
    $A^\modI=A^\modJ$ and $A(x)\to B(f_{R,B}(x))\in\xi(\Ont)$, we have
    $f_{R,B}^\modJ(v)\in B^\modJ=B^\modI$. Hence it suffices to show
    $(v,f_{R,B}^\modJ(v))\in R^\modI$, which follows since
    $f_{R,B}\in\Phi$, $f_{R,B}^\modJ(v)\in B^\modJ$ and
    $R\sqsubseteq^* R$.
  \item Let $A\sqsubseteq\forall R.C\in\Ont$. The claim follows
    analogously to the case for $\exists R.A\sqsubseteq C\in\Ont$.
  \item Let $S\sqsubseteq R\in\Ont$. W.l.o.g., let $(v,w)\in
    S^\modI\setminus S^\modJ$ (if $(v,w)\in S^\modJ$ we immediately
    obtain $(v,w)\in R^\modJ$ since $\pi(S\sqsubseteq
    R)\in\xi(\Ont)$). We show $(v,w)\in R^\modI$. By construction of
    $S^\modI$, there exists some $R'\sqsubseteq^* S$ and $Y$ such that
    either $f_{R',Y}\in\Phi$ and $w=f_{R',Y}^\modJ(v)\in Y^\modJ$ or
    $f_{\mathsf{inv}(R'),Y}\in\Phi$ and
    $v=f_{\mathsf{inv}(R'),Y}^\modJ(w)\in Y^\modJ$. In both cases we
    obtain $(v,w)\in R^\modI$ since $R'\sqsubseteq^* S$ and
    $S\sqsubseteq R\in\Ont$ implies $R'\sqsubseteq^* R$.
  \item Let $A\sqsubseteq{\le}1\,R.B\in\Ont$. Let $v\in A^\modI$,
    $w,u\in B^\modI$, and $(v,w),(v,u)\in R^\modI$. We show
    $w\equality_\modI u$. We distinguish the following subcases:
    \begin{itemize}
    \item $\set{(v,u),(v,w)}\subseteq R^\modJ$. Then the claim is immediate since
      $\pi(A\sqsubseteq{\le}1\,R.B)\in\xi(\Ont)$.
    \item $(v,u)\in R^\modJ$ and $(v,w)\in R^\modI\setminus R^\modJ$.
      By construction of $R^\modI$, there exists some $R'\sqsubseteq^*
      R$ and $Y$ such that either $f_{R',Y}\in\Phi$ and
      $w=f_{R',Y}^\modJ(v)\in Y^\modJ$ or
      $f_{\mathsf{inv}(R'),Y}\in\Phi$ and
      $v=f_{\mathsf{inv}(R'),Y}^\modJ(w)\in Y^\modJ$. In the former
      case, $w\equality_\modI u$ follows since $A(z)\land
      B(f_{R',Y}(z))\land\Inv{R}{z}{x}\land B(x)\to
      f_{R',Y}(z)\equality x\in\xi(\Ont)$. In the latter case, the
      claim follows since $A(f_{\mathsf{inv}(R'),Y}(x)) \wedge B(x)
      \wedge \Inv{R}{f_{\mathsf{inv}(R'),Y}(x)}{y} \wedge B(y)\to x
      \approx y\in\xi(\Ont)$.
    \item $\set{(v,u),(v,w)}\subseteq R^\modI\setminus R^\modJ$. By
      construction of $R^\modI$, there are some
      $R'_1,R'_2\sqsubseteq^* R$ and $Y_1,Y_2$ such that we have one
      of the three following cases:
      \begin{enumerate}
      \item $\set{f_{R'_1,Y_1},f_{R'_2,Y_2}}\subseteq\Phi$,
        $u=f_{R'_1,Y_1}^\modJ(v)\in Y_1^\modJ$ and
        $w=f_{R'_2,Y_2}^\modJ(v)\in Y_2^\modJ$. Then the claim follows
        since $A(z) \wedge B(f_{R'_1,Y_1}(z)) \wedge
        B(f_{R'_2,Y_2}(z)) \to f_{R'_1,Y_1}(z) \approx
        f_{R'_2,Y_2}(z)\in\xi(\Ont)$.
      \item
        $\set{f_{\mathsf{inv}(R'_1),Y_1},f_{R'_2,Y_2}}\subseteq\Phi$,
        $v=f_{\mathsf{inv}(R'_1),Y_1}^\modJ(u)\in Y_1^\modJ$ and
        $w=f_{R'_2,Y_2}^\modJ(f_{\mathsf{inv}(R'_1),Y_1}^\modJ(u))\in
        Y_2^\modJ$. Then the claim follows since
        $A(f_{\mathsf{inv}(R'_1),Y_1}(x)) \wedge B(x) \wedge
        B(f_{R'_2,Y_2}(f_{\mathsf{inv}(R'_1),Y_1}(x))) \to x \approx
        f_{R'_2,Y_2}(f_{\mathsf{inv}(R'_1),Y_1}(x))\in\xi(\Ont)$.
      \item
        $\set{f_{\mathsf{inv}(R'_1),Y_1},f_{\mathsf{inv}(R'_2),Y_2}}\subseteq\Phi$,
        $v=f_{\mathsf{inv}(R'_1),Y_1}^\modJ(u)\in Y_1^\modJ$ and
        $v=f_{\mathsf{inv}(R'_2),Y_2}^\modJ(w)\in Y_2^\modJ$. The
        claim then follows since, as $\modJ$ is a Herbrand model, we
        must have $u=w$ and ${\equality_\modI}$ is reflexive.
        \qedhere
      \end{enumerate}
    \end{itemize}
  \end{itemize}
\end{proof}

\xiroleshorn*

\begin{proof}
  Note that all non-Horn rules in $\xi(\Ont)$ are of type T1, i.e.,
  have unary predicates in the head. Both claims follow from this
  observation and the fact that $\xi(\Ont)$ contains no rules with
  unary predicates in the body and binary predicates in the head
  except for rules of type T6. Thus, whenever a binary predicate $P$
  is disjunctive in $\xi(\Ont)$ (resp., is part of a minimal marking
  of $\xi(\Ont)$), this is due to an axiom $P(x,y)\land x\equality
  z\to P(z,y)$ or $P(x,y)\land y\equality z\to P(x,z)$ in
  $\eqpart{\xi(\Ont)}$ where $\equality$ is disjunctive (resp.,
  marked) in $\xi(\Ont)$. However, predicate $\equality$ cannot be
  part of any marking since then the transitivity rule $x\equality
  y\land y\equality z\to x\equality z$ in $\eqpart{\xi(\Ont)}$ would
  have two marked body atoms.
\end{proof}


%% file: proofs-rewriting.tex
\section{Proofs for Section \ref{sec:rewriting}}
\label{sec:proofs-rewriting}

\xiruleform*

\begin{proof}
  The claim follows by a simple case analysis over the possible rule
  types in $\xi(\Ont)$ (as given in Definition~\ref{def:little-xi}) as
  well as the possible minimal markings for each rule type. The
  analysis exploits that minimal markings involve no binary predicates
  (Proposition~\ref{prop:xi-roles-unmarked}\,\textit{(ii)}).
\end{proof}

\translationback*

\begin{proof}
  By Theorems~\ref{thm:rew-correct-markable} and~\ref{thm:xi-correct},
  it suffices to show that $\Psi(\P)$ is a rewriting of $\P$ whenever
  $\P=\Xi_M(\xi(\Ont))$ for some
  $\Ont$ and $M$. 
  So let $\P$ be as required and let $\Dat$ be a dataset over the
  predicates in $\P$. We show that $\P\cup\Dat$ is satisfiable if
  and only if so is $\pi(\Psi(\P))\cup\Dat$.

  For the direction from left to right, let $\modI$ be a minimal
  Herbrand model of $\P$. We define the interpretation $\modJ$ such
  that
  \begin{itemize}
  \item $\modJ$ coincides with $\modI$ on its domain as well as on
    every concept name, role name, and individual constant in
    $\P\cup\Dat$;
  \item $R_Y^\modJ=\mset{(v,f_{R,Y}^\modI(v))}{v\in\Delta^\modI}$ for
    each function $f_{R,Y}$ in $\P$;
  \item $\tilde R_Y^\modJ=(R_Y^\modJ)^{-1}$ for each role $\tilde R_Y$
    in $\Psi(\P)$;
  \item $S_{\set{R_1,R_2}}^\modJ=R_1^\modJ\cup R_2^\modJ$ for each role $S_{R_1,R_2}$
    in $\Psi(\P)$.
  \end{itemize}
  We next show that $\modJ$ is a model of $\pi(\Psi(\P))\cup\Dat$. By
  construction, $\modJ$ satisfies axioms of type T1--T2 and T4--T6, so
  it suffices to show that $\modJ$ satisfies axioms of type T7--T20:
  \begin{description}
  \item[T7] Let $\overl\bot(z)\land B(x)\land R(x,y)\land
    A(y)\to\bot(z)\in\P$ and $B\sqcap\exists
    R.A\sqsubseteq\bot\in\Psi(\P)$. Let $v\in
    B^\modJ\cap(\exists R.A)^\modJ=B^\modI\cap(\exists R.A)^\modI$. By
    Proposition~\ref{prop:bot-bar-axiomatisation}, we also have
    $v\in\overl\bot^\modI$, and hence
    $v\in\bot^\modI=\bot^\modJ$.
  \item[T8] Let $\overl\bot(z)\land A(f_{R,Y}(x))\land
    B(x)\to\bot(z)\in\P$ and $B\sqcap\exists
    R_Y.A\sqsubseteq\bot\in\Psi(\P)$. Let $v\in B^\modJ\cap(\exists
    R_Y.A)^\modJ$. Then $v\in B^\modI$ and $f_{R,Y}^\modI(v)\in
    A^\modI$. Moreover, by
    Proposition~\ref{prop:bot-bar-axiomatisation}, we have
    $v\in\overl\bot^\modI$, and hence $v\in\bot^\modI=\bot^\modJ$.
  \item[T9] Let $\overl\bot(x)\to\overl\bot(f_{R,Y}(x))\in\P$ and
    $\overl\bot\sqsubseteq\exists R_Y.\overl\bot\in\Psi(\P)$. Let
    $v\in\overl\bot^\modJ=\overl\bot^\modI$. Then
    $f_{R,Y}^\modI(v)\in\overl\bot^\modI$, and hence $v\in(\exists R_Y.\overl\bot)^\modJ$.
  \item[T10] Let $B(x)\to A(f_{R,Y}(x))\in\P$ and $B\sqsubseteq\forall
    R_Y.A\in\Psi(\P)$. Let $v\in B^\modJ=B^\modI$. Then
    $f_{R,Y}^\modI(v)\in A^\modI$, and hence $v\in(\exists
    R_Y.A)^\modJ$. Moreover, since $R_Y^\modJ$ is functional by
    definition, we have $v\in(\forall R_Y.A)^\modJ$.
  \item[T11] Let $B(f_{R,Y}(x))\to A(x)\in\P$ and $\exists
    R_Y.B\sqsubseteq A\in\Psi(\P)$. Let $v\in(\exists
    R_Y.B)^\modJ$. Then $f_{R,Y}^\modI(v)\in B^\modI$, and hence $v\in
    A^\modI=A^\modJ$.
  \item[T12] Let $A(x)\land B(f_{R,Y}(x))\to C(f_{R,Y}(x))\in\P$ and
    $A\sqcap\exists R_Y.B\sqsubseteq\forall R_Y.C\in\Psi(\P)$. Suppose
    $v\in A^\modJ\cap(\exists R_Y.B)^\modJ$. Then $v\in A^\modI$ and
    $f_{R,Y}^\modI(v)\in B^\modI$. Consequently, $f_{R,Y}^\modI(v)\in
    C^\modI=C^\modJ$, and hence $v\in(\exists R_Y.C)^\modJ$. Moreover,
    since $R_Y^\modJ$ is functional by definition, we have
    $v\in(\forall R_Y.C)^\modJ$, as required.
  \item[T13] Let $\overl\bot(z)\land A(x)\land B(f_{R,Y}(x))\land
    C(f_{R,Y}(x))\to\bot(z)\in\P$ and $A\sqcap\exists R_Y(B\sqcap
    C)\sqsubseteq\bot\in\Psi(\P)$. The claim follows similarly to Case
    T8.
  \item[T14] Let $B(f_{R,Y}(x))\land C(f_{R,Y}(x))\to A(x)\in\P$ and
    $\exists R_Y(B\sqcap C)\sqsubseteq A\in\Psi(\P)$. The claim
    follows similarly to Case T11.
  \item[T15] Let $A(z) \wedge B(f_{R',Y}(z)) \wedge \Inv{R}{z}{x}
    \wedge B(x)\to f_{R',Y}(z) \equality x\in\P$ and
    $\set{R'_Y\sqsubseteq S_{\set{R'_Y,R}},R\sqsubseteq
      S_{\set{R'_Y,R}},A\sqsubseteq{\le} 1
      S_{\set{R'_Y,R}}.B}\subseteq\Psi(\P)$ where $R$ occurs in
    $\Ont$. It suffices to show that $\modJ$ satisfies
    $A\sqsubseteq{\le} 1 S_{\set{R'_Y,R}}.B$. Let $v\in
    A^\modJ=A^\modI$, $\set{u,w}\subseteq B^\modJ=B^\modI$, and
    $\set{(v,u),(v,w)}\subseteq S_{\set{R'_Y,R}}^\modJ$. We show
    $u\equality_\modJ w$. We distinguish three cases:
    \begin{itemize}
    \item $\set{(v,u),(v,w)}\subseteq R^\modJ=R^\modI$. Then the claim
      follows since $A(z)\land\Inv{R}{z}{x_1}\land\Inv{R}{z}{x_2}\land
      B(x_1)\land B(x_2)\to x_1\equality x_2\in\P$ and
      ${\equality_\modJ}={\equality_\modI}$.
    \item $(v,u)\in {R'}_Y^\modJ$ and $(v,w)\in R^\modJ=R^\modI$. By
      construction, $u=f_{R',Y}^\modI(v)$, and the claim follows since
      $A(z) \wedge B(f_{R',Y}(z)) \wedge \Inv{R}{z}{x} \wedge B(x)\to
      f_{R',Y}(z) \equality x\in\P$ and
      ${\equality_\modJ}={\equality_\modI}$.
    \item $\set{(v,u),(v,w)}\subseteq {R'}_Y^\modJ$. Since
      ${R'}_Y^\modJ$ is functional by definition, we have $u=w$ and
      hence $u\equality_\modJ w$ by reflexivity of~${\equality_\modJ}$.
    \end{itemize}
  \item[T16] Let $A(f_{R',Y}(x)) \wedge B(x) \wedge
    \Inv{R}{f_{R',Y}(x)}{y}\to x \equality y\in\P$ and $\set{\tilde
      R'_Y\sqsubseteq S_{\set{\tilde R'_Y,R}},R\sqsubseteq
      S_{\set{\tilde R'_Y,R}},A\sqsubseteq{\le} 1 S_{\set{\tilde
          R'_Y,R}}.B,\tilde
      R'_Y\equiv\mathsf{inv}(R'_Y)}\subseteq\Psi(\P)$ where $R$ occurs
    in $\Ont$. It suffices to show that $\modJ$ satisfies
    $A\sqsubseteq{\le} 1 S_{\set{\tilde R'_Y,R}}.B$. Let $v\in
    A^\modJ=A^\modI$, $\set{u,w}\subseteq B^\modJ=B^\modI$, and
    $\set{(v,u),(v,w)}\subseteq S_{\set{\tilde R'_Y,R}}^\modJ$. We show
    $u\equality_\modJ w$. We distinguish three cases:
    \begin{itemize}
    \item $\set{(v,u),(v,w)}\subseteq R^\modJ=R^\modI$. Then the claim
      follows since $A(z)\land\Inv{R}{z}{x_1}\land\Inv{R}{z}{x_2}\land
      B(x_1)\land B(x_2)\to x_1\equality x_2\in\P$ and
      ${\equality_\modJ}={\equality_\modI}$.
    \item $(v,u)\in {\tilde R}^{\prime\modJ}_Y$ and $(v,w)\in
      R^\modJ=R^\modI$. By construction, $v=f_{R',Y}^\modI(u)$, and
      the claim follows since $A(f_{R',Y}(x)) \wedge B(x) \wedge
      \Inv{R}{f_{R',Y}(x)}{y}\to x \equality y\in\P$ and
      ${\equality_\modJ}={\equality_\modI}$.
    \item $\set{(v,u),(v,w)}\subseteq {\tilde R}^{\prime\modJ}_Y$. Since
      $\modI$ is a Herbrand model, $f_{R',Y}^\modI$ is injective and
      hence ${\tilde R}^{\prime\modJ}_Y$ is functional. Therefore, we have
      $u=w$ and hence $u\equality_\modJ w$ by reflexivity
      of~${\equality_\modJ}$.
    \end{itemize}
  \item[T17] Let $A(z) \wedge B(f_{R,Y}(z)) \wedge B(f_{R',Z}(z))\to
    f_{R,Y}(z) \equality f_{R',Z}(z)\in\P$ and $\set{R_Y\sqsubseteq
      S_{\set{R_Y,R'_Z}},R'_Z\sqsubseteq
      S_{\set{R_Y,R'_Z}},A\sqsubseteq{\le} 1
      S_{\set{R_Y,R'_Z}}.B}\subseteq\Psi(P)$. It suffices to show
    that $\modJ$ satisfies $A\sqsubseteq{\le} 1
    S_{\set{R_Y,R'_Z}}.B$. Let $v\in A^\modJ=A^\modI$,
    $\set{u,w}\subseteq B^\modJ=B^\modI$, and
    $\set{(v,u),(v,w)}\subseteq S_{\set{R_Y,R'_Z}}^\modJ$. We
    show $u\equality_\modJ w$. W.l.o.g., we distinguish two cases:
    \begin{itemize}
    \item $\set{(v,u),(v,w)}\subseteq R_Y^\modJ$. Since $R_Y^\modJ$ is
      functional by definition, we have $u=w$ and hence
      $u\equality_\modJ w$ by reflexivity of~${\equality_\modJ}$.
    \item $(v,u)\in {R}^\modJ_Y$ and $(v,w)\in {R'}_Z^\modJ$. Then, by
      construction, $u=f_{R,Y}(v)$ and $w=f_{R',Z}(v)$. The claim
      follows since $A(z) \wedge B(f_{R,Y}(z)) \wedge
      B(f_{R',Z}(z))\to f_{R,Y}(z) \equality f_{R',Z}(z)\in\P$ and
      ${\equality_\modJ}={\equality_\modI}$.
    \end{itemize}
  \item[T18] Let $A(f_{R,Y}(x)) \wedge B(x) \wedge
    B(f_{R',Z}(f_{R,Y}(x)))\to x \approx f_{R',Z}(f_{R,Y}(x))\in\P$
    and $\set{\tilde R_Y\sqsubseteq S_{\set{\tilde
          R_Y,R'_Z}},R'_Z\sqsubseteq S_{\set{\tilde
          R_Y,R'_Z}},A\sqsubseteq{\le} 1 S_{\set{\tilde
          R_Y,R'_Z}}.B,\tilde
      R_Y\equiv\mathsf{inv}(R_Y)}\subseteq\Psi(P)$. It suffices to
    show that $\modJ$ satisfies $A\sqsubseteq{\le} 1 S_{\set{\tilde
        R_Y,R'_Z}}.B$. Let $v\in A^\modJ=A^\modI$, $\set{u,w}\subseteq
    B^\modJ=B^\modI$, and $\set{(v,u),(v,w)}\subseteq S_{\set{\tilde
        R_Y,R'_Z}}^\modJ$. We show $u\equality_\modJ w$. We
    distinguish three cases:
    \begin{itemize}
    \item $\set{(v,u),(v,w)}\subseteq\tilde R_Y^\modJ$. Since $\modI$
      is a Herbrand model, $f_{R,Y}^\modI$ is injective and hence
      $\tilde R_Y^\modJ$ is functional. Therefore, we have $u=w$ and
      hence $u\equality_\modJ w$ by reflexivity
      of~${\equality_\modJ}$.
    \item $(v,u)\in {\tilde R}^\modJ_Y$ and $(v,w)\in
      {R'}_Z^\modJ$. Then, by construction, $v=f_{R,Y}(u)$ and
      $w=f_{R',Z}(f_{R,Y}(u))$. The claim follows since $A(f_{R,Y}(x))
      \wedge B(x) \wedge B(f_{R',Z}(f_{R,Y}(x)))\to x \approx
      f_{R',Z}(f_{R,Y}(x))\in\P$ and
      ${\equality_\modJ}={\equality_\modI}$.
    \item $\set{(v,u),(v,w)}\subseteq{R'}_Z^\modJ$. Since
      ${R'}_Z^\modJ$ is functional by definition, we have $u=w$ and hence
      $u\equality_\modJ w$ by reflexivity of~${\equality_\modJ}$.
    \end{itemize}
  \item[T19] Let $R(x,y)\to\overl\bot(x)\in\P$ and $\exists
    R.\top\sqsubseteq\overl\bot\in\Psi(\P)$. Let $v\in(\exists
    R.\top)^\modJ$. Then, for some $w$, $(v,w)\in
    R^\modJ=R^\modI$. Consequently,
    $v\in\overl\bot^\modI=\overl\bot^\modJ$.
  \item[T20] Let $R(x,y)\to\overl\bot(y)\in\P$ and
    $\top\sqsubseteq\forall R.\overl\bot\in\Psi(\P)$. Let $(v,w)\in
    R^\modJ=R^\modI$. Then $w\in\overl\bot^\modI=\overl\bot^\modJ$.
  \end{description}

  For the direction from right to left, let $\modJ$ be a minimal
  Herbrand model of $\pi(\Psi(\P))\cup\Dat$. We define the
  interpretation $\modI$ such that
  \begin{itemize}
  \item $\modI$ coincides with $\modJ$ on its domain as well as on every concept name, role
    name, and individual constant in $\P\cup\Dat$;
  \item $f_{R,Y}^\modI(v)\in\mset{w\in\overl\bot^\modJ}{(v,w)\in
      R_Y^\modJ}$ for each function $f_{R,Y}$ in $\P$ (note that by
    Proposition~\ref{prop:bot-bar-axiomatisation} and the fact that
    $\overl\bot\sqsubseteq\exists R_Y.\overl\bot\in\Psi(\P)$, the set
    $\mset{w\in\overl\bot^\modJ}{(v,w)\in R_Y^\modJ}$ is nonempty for
    every $v\in\Delta^\modJ$).
  \end{itemize}
  We show that $\modI$ is a model of $\P\cup\Dat$. By construction,
  $\modI$ satisfies rules of type T1--T2 and T4--T6, so
  it suffices to show that $\modJ$ satisfies rules of type T7--T20:
  \begin{description}
  \item[T7] Let $\overl\bot(z)\land B(x)\land R(x,y)\land
    A(y)\to\bot(z)\in\P$ and $B\sqcap\exists
    R.A\sqsubseteq\bot\in\Psi(\P)$. Let $w\in\overl\bot^\modI$, $v\in
    B^\modI$, $v'\in A^\modI$, and $(v,v')\in R^\modI$. Then
    $v\in(B\sqcap\exists R.A)^\modJ$ and hence $v\in\bot^J$. Since
    $\bot(x)\to\square\in\pi(\Psi(\P))$, this means that $\modJ$ is
    not a model of $\pi(\Psi(\P))\cup\Dat$, so the claim holds
    vacuously.
  \item[T8] Let $\overl\bot(z)\land A(f_{R,Y}(x))\land
    B(x)\to\bot(z)\in\P$ and $B\sqcap\exists
    R_Y.A\sqsubseteq\bot\in\Psi(\P)$. Let $w\in\overl\bot^\modI$,
    $v\in B^\modI$, and $f_{R,Y}^\modI(v)\in A^\modI$. Then
    $(v,f_{R,Y}^\modI(v))\in R_Y^\modJ$. Thus, $v\in(B\sqcap\exists
    R_Y.A)^\modJ$, and so $v\in\bot^\modJ$. Since
    $\bot(x)\to\square\in\pi(\Psi(\P))$, this means that $\modJ$ is
    not a model of $\pi(\Psi(\P))\cup\Dat$, so the claim holds
    vacuously.
  \item[T9] Let $\overl\bot(x)\to\overl\bot(f_{R,Y}(x))\in\P$ and
    $\overl\bot\sqsubseteq\exists R_Y.\overl\bot\in\Psi(\P)$. Let
    $v\in\overl\bot^\modI=\overl\bot^\modJ$. Then $v\in(\exists
    R_Y.\overl\bot)^\modJ$, hence the set
    $\mset{w\in\overl\bot^\modJ}{(v,w)\in R_Y^\modJ}$ is nonempty and
    $f_{R,Y}^\modI(v)\in\overl\bot^\modJ=\overl\bot^\modI$.
  \item[T10] Let $B(x)\to A(f_{R,Y}(x))\in\P$ and $B\sqsubseteq\forall
    R_Y.A\in\Psi(\P)$. Let $v\in B^\modI=B^\modJ$. Then $v\in(\forall
    R_Y.A)^\modJ$. Since $f_{R,Y}^\modI(v)\in\mset{w}{(v,w)\in
      R_Y^\modJ}$, it follows that $f_{R,Y}^\modI(v)\in A^\modJ=A^\modI$.
  \item[T11] Let $B(f_{R,Y}(x))\to A(x)\in\P$ and $\exists
    R_Y.B\sqsubseteq A\in\Psi(\P)$. Let $f_{R,Y}^\modI(v)\in
    B^\modI=B^\modJ$. Since $(v,f_{R,Y}^\modI(v))\in R_Y^\modJ$, we
    have $v\in(\exists R_Y.B)^\modJ$, and hence $v\in
    A^\modJ=A^\modI$.
  \item[T12] Let $A(x)\land B(f_{R,Y}(x))\to C(f_{R,Y}(x))\in\P$ and
    $A\sqcap\exists R_Y.B\sqsubseteq\forall R_Y.C\in\Psi(\P)$. Let
    $v\in A^\modI=A^\modJ$ and $f_{R,Y}^\modI(v)\in
    B^\modI=B^\modJ$. Then $(v,f_{R,Y}^\modI(v))\in R_Y^\modJ$, and
    consequently $v\in(\exists R_Y.B)^\modJ$. Since this implies
    $v\in(\forall R_Y.C)^\modJ$ and $(v,f_{R,Y}^\modI(v))\in
    R_Y^\modJ$, we then obtain $f_{R,Y}^\modI(v)\in C^\modJ=C^\modI$, as required.
  \item[T13] Let $\overl\bot(z)\land A(x)\land B(f_{R,Y}(x))\land
    C(f_{R,Y}(x))\to\bot(z)\in\P$ and $A\sqcap\exists R_Y(B\sqcap
    C)\sqsubseteq\bot\in\Psi(\P)$. The claim follows similarly to Case
    T8.
  \item[T14] Let $B(f_{R,Y}(x))\land C(f_{R,Y}(x))\to A(x)\in\P$ and
    $\exists R_Y(B\sqcap C)\sqsubseteq A\in\Psi(\P)$. The claim
    follows similarly to Case T11.
  \item[T15] Let $A(z) \wedge B(f_{R',Y}(z)) \wedge \Inv{R}{z}{x}
    \wedge B(x)\to f_{R',Y}(z) \equality x\in\P$ and
    $\set{R'_Y\sqsubseteq S_{\set{R'_Y,R}},R\sqsubseteq
      S_{\set{R'_Y,R}},A\sqsubseteq{\le} 1
      S_{\set{R'_Y,R}}.B}\subseteq\Psi(\P)$. Let $v\in A^\modI=A^\modJ$, $f_{R',Y}^\modI(v)\in
    B^\modI=B^\modJ$, $(v,w)\in R^\modI=R^\modJ$, and $w\in
    B^\modI=B^\modJ$. We show $f_{R',Y}^\modI(v)\equality_\modI w$. By
    construction, we have $(v,f_{R',Y}^\modI(v))\in
    {R'}_Y^\modJ$. Since $\set{R'_Y\sqsubseteq
      S_{\set{R'_Y,R}},R\sqsubseteq
      S_{\set{R'_Y,R}}}\subseteq\Psi(\P)$, we thus have
    $\set{(v,f_{R',Y}^\modI(v)),(v,w)}\subseteq
    S_{\set{R'_Y,R}}^\modJ$. Since $A\sqsubseteq{\le} 1
    S_{\set{R'_Y,R}}.B\in\Psi(\P)$, we conclude
    $f_{R',Y}^\modI(v)\equality_\modJ w$ and hence
    $f_{R',Y}^\modI(v)\equality_\modI w$.
  \item[T16] Let $A(f_{R',Y}(x)) \wedge B(x) \wedge
    \Inv{R}{f_{R',Y}(x)}{y}\to x \equality y\in\P$ and $\set{\tilde
      R'_Y\sqsubseteq S_{\set{\tilde R'_Y,R}},R\sqsubseteq
      S_{\set{\tilde R'_Y,R}},A\sqsubseteq{\le} 1 S_{\set{\tilde
          R'_Y,R}}.B,\tilde
      R'_Y\equiv\mathsf{inv}(R'_Y)}\subseteq\Psi(\P)$. Let
    $f_{R',Y}(v)\in A^\modI=A^\modJ$, $v\in B^\modI=B^\modJ$,
    $(f_{R',Y}(v),w)\in R^\modI=R^\modJ$, and $w\in
    B^\modI=B^\modJ$. We show $v\equality_\modI w$. By construction,
    we have $(v,f_{R',Y}^\modI(v))\in {R'}_Y^\modJ$, and hence
    $(f_{R',Y}^\modI(v),v)\in {\tilde R^{\prime\modJ}}_Y$. Since
    $\set{\tilde R'_Y\sqsubseteq S_{\set{\tilde R'_Y,R}},R\sqsubseteq
      S_{\set{\tilde R'_Y,R}}}\subseteq\Psi(\P)$, we thus have
    $\set{(f_{R',Y}^\modI(v),v),(f_{R',Y}^\modI(v),w)}\subseteq
    S_{\set{\tilde R'_Y,R}}^\modJ$. Since $A\sqsubseteq{\le} 1
    S_{\set{\tilde R'_Y,R}}.B\in\Psi(\P)$, we conclude
    $v\equality_\modJ w$ and hence $v\equality_\modI w$.
  \item[T17] Let $A(z) \wedge B(f_{R,Y}(z)) \wedge B(f_{R',Z}(z))\to
    f_{R,Y}(z) \equality f_{R',Z}(z)\in\P$ and $\set{R_Y\sqsubseteq
      S_{\set{R_Y,R'_Z}},R'_Z\sqsubseteq
      S_{\set{R_Y,R'_Z}},A\sqsubseteq{\le} 1
      S_{\set{R_Y,R'_Z}}.B}\subseteq\Psi(P)$. Let $v\in
    A^\modI=A^\modJ$ and
    $\set{f_{R,Y}^\modI(v),f_{R',Z}^\modI(v)}\subseteq
    B^\modI=B^\modJ$. We show $f_{R,Y}^\modI(v)\equality_\modI
    f_{R',Z}^\modI(v)$. By construction, we have
    $(v,f_{R,Y}^\modI(v))\in R_Y^\modJ$ and $(v,f_{R',Z}^\modI(v))\in
    {R'}_Z^\modJ$. Since $\set{R_Y\sqsubseteq
      S_{\set{R_Y,R'_Z}},R'_Z\sqsubseteq
      S_{\set{R_Y,R'_Z}}}\subseteq\Psi(\P)$, we thus have
    $\set{(v,f_{R,Y}^\modI(v)),(v,f_{R',Z}^\modI(v))}\subseteq
    S_{\set{R_Y,R'_Z}}^\modJ$. Since $A\sqsubseteq{\le} 1
    S_{\set{R_Y,R'_Z}}.B\in\Psi(\P)$, we conclude
    $f_{R,Y}^\modI(v)\equality_\modJ f_{R',Z}^\modI(v)$ and hence
    $f_{R,Y}^\modI(v)\equality_\modI f_{R',Z}^\modI(v)$.
  \item[T18] Let $A(f_{R,Y}(x)) \wedge B(x) \wedge
    B(f_{R',Z}(f_{R,Y}(x)))\to x \approx f_{R',Z}(f_{R,Y}(x))\in\P$
    and $\set{\tilde R_Y\sqsubseteq S_{\set{\tilde
          R_Y,R'_Z}},R'_Z\sqsubseteq S_{\set{\tilde
          R_Y,R'_Z}},A\sqsubseteq{\le} 1 S_{\set{\tilde
          R_Y,R'_Z}}.B,\tilde
      R_Y\equiv\mathsf{inv}(R_Y)}\subseteq\Psi(P)$. Let
    $f_{R,Y}^\modI(v)\in A^\modI=A^\modJ$, $v\in B^\modI=B^\modJ$, and
    $f_{R',Z}^\modI(f_{R,Y}^\modI(v))\in B^\modI=B^\modJ$. We show
    $v\equality_\modI f_{R',Z}^\modI(f_{R,Y}^\modI(v))$. By
    construction, we have $(v,f_{R,Y}^\modI(v))\in R_Y^\modJ$ and
    hence $(f_{R,Y}^\modI(v),v)\in\tilde R_Y^\modJ$. Moreover,
    $(f_{R,Y}^\modI(v),f_{R',Z}^\modI(f_{R,Y}^\modI(v)))\in
    {R'}_Z^\modJ$. Since $\set{\tilde R_Y\sqsubseteq S_{\set{\tilde
          R_Y,R'_Z}},R'_Z\sqsubseteq S_{\set{\tilde
          R_Y,R'_Z}}}\subseteq\Psi(\P)$, we thus have
    $\set{(f_{R,Y}^\modI(v),v),(f_{R,Y}^\modI(v),f_{R',Z}^\modI(f_{R,Y}^\modI(v)))}\subseteq
    S_{\set{\tilde R_Y,R'_Z}}^\modJ$. Since $A\sqsubseteq{\le} 1
    S_{\set{\tilde R_Y,R'_Z}}.B\in\Psi(\P)$, we conclude
    $v\equality_\modJ f_{R',Z}^\modI(f_{R,Y}^\modI(v))$ and hence
    $v\equality_\modI f_{R',Z}^\modI(f_{R,Y}^\modI(v))$.
  \item[T19] Let $R(x,y)\to\overl\bot(x)\in\P$ and $\exists
    R.\top\sqsubseteq\overl\bot\in\Psi(\P)$. Let $(v,w)\in
    R^\modI=R^\modJ$. Since $R(x,y)\to\top(y)\in\pi(\Psi(\P))$, we
    then have $v\in(\exists R.\top)^\modJ$, and hence
    $v\in\overl\bot^\modJ=\overl\bot^\modI$.
  \item[T20] Let $R(x,y)\to\overl\bot(y)\in\P$ and
    $\top\sqsubseteq\forall R.\overl\bot\in\Psi(\P)$. Let $(v,w)\in
    R^\modI=R^\modJ$. Since $R(x,y)\to\top(x)\in\pi(\Psi(\P))$, we
    then have $v\in\top^\modJ$, and hence $v\in(\forall
    R.\overl\bot)^\modJ$. Since $(v,w)\in R^\modJ$, we thus obtain
    $w\in\overl\bot^\modJ=\overl\bot^\modI$.
    \qedhere
  \end{description}
\end{proof}

\finalrewritability*

\begin{proof}
  The claims follow from the observation that $\Psi(\Xi_M(\xi(\Ont)))$
  only introduces new axioms of type T1-T2, T4, T7-T14 and T19-T20 to
  $\xi(\Ont)$ unless $\Ont$ contains functionality
  assertions. Moreover, none of the axioms in
  $\Psi(\Xi_M(\xi(\Ont)))\setminus\xi(\Ont)$ contains inverse roles
  unless so does $\xi(\Ont)$. Thus, all axioms in
  $\Psi(\Xi_M(\xi(\Ont)))\setminus\xi(\Ont)$ are
  \begin{itemize}
  \item in Horn-$\mathcal{ALC}$ if $\Ont$ is between $\mathcal{ELU}$ and $\mathcal{ALCH}$;
  \item in Horn-$\mathcal{ALCI}$ if $\Ont$ is in $\mathcal{ALCI}$ or $\mathcal{ALCHI}$;
  \item in Horn-$\mathcal{ALCH}$ if $\Ont$ is in $\mathcal{ALCF}$ or $\mathcal{ALCHF}$;
  \item in Horn-$\mathcal{ALCHI}$ if $\Ont$ is in $\mathcal{ALCIF}$ or $\mathcal{ALCHIF}$.
  \end{itemize}
  The claims immediately follow.
\end{proof}

\section{Proofs for Section \ref{sec:complexity-results}}
\label{sec:proofs-complexity}

\elusatexptimehard*

\begin{proof}
  We prove the claim by adapting the \EXPTIME-hardness argument for
  Horn-\ALC{} by \citeA{KroetzschRH13}. Given a polynomially
  space-bounded alternating Turing machine $\mathcal{M}$ and a word
  $w$, \citeauthor{KroetzschRH13}\ construct a Horn-\ALC{} ontology
  $\Ont_{\mathcal{M},w}$ such that $\mathcal{M}$ accepts a $w$ if and
  only if $\Ont_{\mathcal{M},w}\models I_w\sqsubseteq A$ for a concept
  name $A$ and a conjunction $I_w$ that encodes $w$. Equivalently,
  $\mathcal{M}$ accepts $w$ if and only if
  $\Ont_{\mathcal{M},w}\cup\set{I_w\sqcap A\sqsubseteq\bot}$ is
  unsatisfiable.

  Ontology $\Ont_{\mathcal{M},w}$ is in \EL{} except for axioms of the
  form $H\sqcap C\sqsubseteq\forall S.C$. We will now encode all such
  axioms into $\ELU$. Let $not\_C$ be a fresh concept name for every
  atomic concept $C$ in $\Ont_{\mathcal{M},w}$, and let
  $\Ont'_{\mathcal{M},w}$ be obtained from $\Ont_{\mathcal{M},w}$ by
  replacing every axiom of the form $H\sqcap C\sqsubseteq\forall S.C$
  by the axioms $H\sqcap C\sqcap\exists S.not\_C\sqsubseteq\bot$,
  $C\sqcap not\_C\sqsubseteq\bot$ and $\top\sqsubseteq C\sqcup
  not\_C$. Clearly, $\Ont'_{\mathcal{M},w}\cup\set{I_w\sqcap
    A\sqsubseteq\bot}$ is in \ELU{} and
  $\Ont'_{\mathcal{M},w}\cup\set{I_w\sqcap A\sqsubseteq\bot}$ is
  satisfiable if and only if so is
  $\Ont_{\mathcal{M},w}\cup\set{I_w\sqcap A\sqsubseteq\bot}$. The
  claim follows since the set $\mset{not\_C}{\top\sqsubseteq C\sqcup
    not\_C\in\Ont'_{\mathcal{M},w}}$ is a marking of
  $\Ont'_{\mathcal{M},w}\cup\set{I_w\sqcap A\sqsubseteq\bot}$.
\end{proof}

\complexity*

\begin{proof}
  The claim follows by Theorem~\ref{thm:final-rewritability},
  Lemma~\ref{lem:elu-sat-exptime-hard} and the results for logics
  between Horn-$\mathcal{ALC}$ and Horn-$\mathcal{ALCHIF}$
  in~\cite{KroetzschRH13}.
\end{proof}
